\documentclass[journal]{IEEEtran}
\usepackage[cm]{fullpage}
\usepackage{amsthm}

\usepackage{cite}
\usepackage{fixltx2e}

\usepackage{graphicx}

\usepackage{subcaption}              

\usepackage[utf8]{inputenc}

\usepackage{amsfonts}
\usepackage{amsmath}
\usepackage{mathtools}

\usepackage{amssymb}

\usepackage{bm}

\usepackage{color,verbatim}
\usepackage{multirow}
\usepackage{accents}

\usepackage{theoremref}

\usepackage{flushend}

\usepackage{url}

\newcounter{rulecounter}
\newcommand{\resetrule}{ \setcounter{rulecounter}{0}}
\resetrule

\newsavebox{\selvestebox}
\newenvironment{colbox}[1]
  {\newcommand\colboxcolor{#1}%
   \begin{lrbox}{\selvestebox}%
   \begin{minipage}{\dimexpr\columnwidth-2\fboxsep\relax}}
  {\end{minipage}\end{lrbox}%
   \begin{center}
   \colorbox{\colboxcolor}{\usebox{\selvestebox}}
   \end{center}}

\definecolor{orange}{rgb}{1,0.8,0}
\definecolor{gray}{rgb}{.9,0.9,0.9}
\definecolor{darkgray}{rgb}{.3,0.3,0.3}
\definecolor{darkblue}{rgb}{.1,0.0,0.3}
\definecolor{lightblue}{rgb}{0.7,0.7,1}
\definecolor{lightred}{rgb}{1,0.7,.7}
\definecolor{purple}{RGB}{204,153,255}
\definecolor{lightgray}{rgb}{.95,0.95,0.95}
\definecolor{lightgreen}{rgb}{0.3,0.5,0.3}
\definecolor{darkgreen}{rgb}{0.05,0.3,0.05}

\newcommand{\ra}{$\rightarrow$~}

\newcommand{\bbm}[1]{{\bar{\bm #1}}}

\newcommand{\tbm}[1]{{\tilde{\bm #1}}}

\newcommand{\hbm}[1]{{\hat{\bm #1}}}
\newcommand{\inv}{^{-1}}

\newcommand{\rfield}{\mathbb{R}}

\newcommand{\cov}[2]{\mathop{\rm Cov}\left\{#1,#2\right\}}
\newcommand{\diag}[1]{\mathop{\rm diag}\brackets{#1}}

\newcommand{\tr}[1]{\mathop{\rm Tr}\left(#1\right)}

 \newcommand{\define}{\triangleq}

\newcommand{\expected}[1]{\mathop{\textrm{E}}\brackets{#1} }
\newcommand{\expectednb}{\mathop{\textrm{E}}\nolimits}

\newtheorem{myproposition}{Proposition}
\newtheorem{myremark}{Remark}
\newtheorem{myproblemstatement}{Problem Statement}

\newtheorem{mytheorem}{Theorem}
\newtheorem{mydefinition}{Definition}
\newtheorem{mycorollary}{Corollary}

\newtheorem{lemma}{Lemma}

\newcounter{exampleind}
\newcommand{\numberedexample}[1]{\noindent\textbf{Example\refstepcounter{exampleind} \label{#1} \theexampleind } }

\newcounter{remarkind}
\newcommand{\numberedremark}[1]{\noindent\textbf{Remark\refstepcounter{remarkind} \label{#1} \theremarkind}}

\renewcommand{\define}{:=}
\newcommand{\columnspan}{\hc{\mathcal{R}}}
\DeclareMathOperator*{\argmin}{arg\,min}
\DeclareMathOperator*{\argmax}{arg\,max}
\renewcommand{\expected}[1]{\expectednb\left[#1\right]}  
\renewcommand{\expectednb}{\hc{\mathbb{E}}}
\newcommand{\indicator}{\hc{\mathcal{I}}}

\newcommand{\vertexset}{\hc{\mathcal{V}}}
\newcommand{\edgeset}{\hc{\mathcal{E}}}
\newcommand{\vertexind}{{\hc{{n}}}}
\newcommand{\vertexindp}{{\hc{{n}'}}} 
\newcommand{\vertexindpp}{{\hc{{n}''}}} 
\newcommand{\vertexnum}{{\hc{{N}}}}
\newcommand{\weightfun}{\hc{w}} 
 
\newcommand{\featurevecdim}{{\hc{D}} }
\newcommand{\adjacencymat}{\hc{\bm W}} 
\newcommand{\laplacianevecmat}{\hc{\bm U}} 
\newcommand{\uslaplacianevecmat}{\hc{\tbm U}}  
\newcommand{\laplacianevalmat}{\hc{\bm \Lambda}} 
\newcommand{\uslaplacianevalmat}{\hc{\tbm \Lambda}} 
\newcommand{\iblaplacianevalmat}{\hc{\bm \Lambda_B}}  
\newcommand{\oblaplacianevalmat}{\hc{\bm \Lambda_B^c}}  
\newcommand{\laplacianevec}{\hc{\bm u}} 
\newcommand{\laplacianefun}{\hc{u}}  
\newcommand{\oblaplacianevecmat}{\hc{{\bm U_B^c}}}  

\newcommand{\signalfun}{\hc{f_0}} 
\newcommand{\signalvec}{\hc{\bm f_0}}

\newcommand{\signalestfun}{\hc{\hat f_0}} 
\newcommand{\signalestvec}{\hc{\hbm f_0}} 
\newcommand{\signalridgeestvec}{\hc{\hbm f}_\text{RR}} 
\newcommand{\signalridgeestfun}{\hc{\hat f}_\text{RR}} 
\newcommand{\signalridgesmoothvec}{\hc{\hbm f}_\text{RRS}} 
\newcommand{\signallmmseestvec}{\hc{\hbm f}_\text{LMMSE}} 

\newcommand{\samplingset}{\hc{\mathcal{S}}} 
\newcommand{\samplingmat}{\hc{\bm \Phi}} 
\newcommand{\sampleind}{{\hc{s}}} 
\newcommand{\samplenum}{{\hc{S}}} 
\newcommand{\observation}{\hc{y}} 
\newcommand{\observationvec}{\hc{\bm y}} 
 
\newcommand{\fourierobservation}{\hc{\tilde y}} 
\newcommand{\noisesamp}{{\hc{e}}} 
\newcommand{\noisevec}{\hc{\bm \noisesamp}} 
\newcommand{\noisevar}{\hc{ {\sigma^2_\noisesamp}}}  %
\newcommand{\freqind}{\hc{m}}  %
\newcommand{\freqindp}{\hc{m'}}  %

\newcommand{\lossfun}{\hc{\mathcal{L}}} 
\newcommand{\rkhs}{{\hc{\mathcal{H}}}}
\newcommand{\rkhsfunsymbol}{f}
\newcommand{\rkhsvec}{\hc{\bm \rkhsfunsymbol}} 
\newcommand{\rkhsfun}{\hc{\rkhsfunsymbol}} 
\newcommand{\rkhsestfun}{\hc{\hat \rkhsfunsymbol}} 
 
\newcommand{\fourierrkhsfun}{\hc{\tilde \rkhsfunsymbol}} 
\newcommand{\fourierrkhsvec}{\hc{\tbm \rkhsfunsymbol}} 
 
\newcommand{\graphvariation}{\hc{\partial}}
\newcommand{\kernelmap}{\hc{\kappa}} 
\newcommand{\fullkernelmat}{\hc{\bbm K}} 
\newcommand{\fullkernelvec}{\hc{\bbm \kappa}} 
\newcommand{\fullkernelevalmat}{\hc{\bm \Lambda_{\fullkernelmat}}} 
\newcommand{\fullkernelevecmat}{\hc{\bm U_{\fullkernelmat}}} 
\newcommand{\samplekernelmat}{\hc{\bm K}} 
\newcommand{\fullalpha}{\hc{\bar \alpha}} 
\newcommand{\fullalphavec}{\hc{\bbm \alpha}} 
\newcommand{\fullalphaestvec}{\hc{\hat{\bbm \alpha}}} 
\newcommand{\samplealpha}{\hc{\alpha}} 
\newcommand{\samplealphavec}{\hc{\bm \alpha}}
\newcommand{\samplealphaestvec}{\hc{\hbm \alpha}}

\newcommand{\blset}{{\hc{\mathcal{B}}}} 
\newcommand{\blnum}{{\hc{B}}} 
\newcommand{\blfouriervec}{\hc{{\tbm f}_\blset}} %
\newcommand{\blselectionmat}{\hc{\bm \Psi}} 
\newcommand{\compblselectionmat}{\hc{\bm \Psi_c}}  
\newcommand{\bllaplacianevecmat}{\hc{\laplacianevecmat_{\blset}}} 
\newcommand{\signallsestvec}{\hc{\hbm f}_\text{LS}} 
\newcommand{\blfullkernelmat}{\hc{\bbm K_\beta}} 

\renewcommand{\cov}{\hc{c}}  %
\newcommand{\covvec}{\hc{\bm c}}  %
\newcommand{\covmat}{\hc{\bm C}}  %
\newcommand{\covmatest}{\hc{\hbm C}}  %
\newcommand{\covcond}{\hc{c_{1|\rest}}}  %
\newcommand{\rest}{\hc{2:\vertexnum}}  %
\newcommand{\invcovmatent}{\hc{\gamma}}  
\newcommand{\rkhsnum}{{\hc{M}}}
\newcommand{\rkhsind}{{\hc{m}}}
\newcommand{\trsamplealphavec}{\hc{\check{\samplealphavec}}} 
\newcommand{\fourieralphavec}{\hc{\tilde{\samplealphavec}}} 
\newcommand{\fourieralpha}{\hc{{\tilde \alpha}}} 
\newcommand{\softthresholding}{{\hc{\mathcal{T}}}}
\newcommand{\radius}{{\hc{R}}}
\newcommand{\kernelcoef}{{\hc{ \theta}}}
\newcommand{\kernelcoefvec}{{\hc{\bm \theta}}}

\newcommand{\auxveco}{{\hc{\bm \beta}}} 
\newcommand{\auxvect}{{\hc{\bm \xi}}} 
\newcommand{\auxentt}{{\hc{\xi}}} 
\newcommand{\auxmat}{\hc{\bm G}}  %
\newcommand{\auxmatt}{\hc{\bm M}}  
\newcommand{\auxmatf}{\hc{\bm \Upsilon}}  
\newcommand{\perpvec}{\hc{\bm \beta}}
\newcommand{\perpestvec}{\hc{\hbm \beta}}
\newcommand{\rotationmat}{\hc{\bm R}}  %
\newcommand{\diagvec}{\hc{\bm d}}  %
\newcommand{\diagent}{\hc{d}}  
\newcommand{\fourierdiagent}{\hc{D}}  %

\newcommand{\smoothfun}{\hc{\tilde r}}  %
\newcommand{\fourierfiltdiag}{\hc{\tilde h}}  %
\newcommand{\filtercoef}{\hc{h}}  %

\newcommand{\iternot}[1]{^{\hc{(}{#1}\hc{)}}}  
\newcommand{\iternotT}[1]{^{\hc{(}{#1}\hc{)},T}}  
\newcommand{\regpar}{\hc{\mu}}

\usepackage{algorithmic}
\usepackage[ruled,vlined,resetcount]{algorithm2e}

\allowdisplaybreaks
\begin{document}
\title{Kernel-based Reconstruction of Graph Signals} %
\author{Daniel Romero, \emph{Member, IEEE}, Meng Ma, Georgios
  B. Giannakis, \emph{Fellow, IEEE} \thanks{This work was supported by
    ARO grant W911NF-15-1-0492 and NSF grants 1343248, 1442686, and
    1514056.

 The authors are with the Dept. of ECE and the Digital Tech. Center,
 Univ. of Minnesota, USA. E-mail: \{dromero,maxxx971,georgios\}@umn.edu.
}
}
\maketitle


 \newcommand{\cmt}[1]{} 
 \newcommand{\hc}[1]{\textcolor{black}{#1}} 

\begin{abstract}

  A number of applications in engineering, social sciences, physics,
  and biology involve inference over networks. In this context, graph
  signals are widely encountered as descriptors of vertex attributes
  or features in graph-structured data.  Estimating such signals in
  all vertices given noisy observations of their values on a subset of
  vertices has been extensively analyzed in the literature of signal
  processing on graphs (SPoG).  This paper advocates kernel regression
  as a framework generalizing popular SPoG modeling and reconstruction
  and expanding their capabilities. Formulating signal reconstruction
  as a regression task on reproducing kernel Hilbert spaces of graph
  signals permeates benefits from statistical learning, offers fresh
  insights, and allows for estimators to leverage richer forms of
  prior information than existing alternatives. A number of SPoG
  notions such as bandlimitedness, graph filters, and the graph
  Fourier transform are naturally accommodated in the kernel
  framework. Additionally, this paper capitalizes on the so-called
  representer theorem to devise simpler versions of existing Thikhonov
  regularized estimators, and offers a novel probabilistic
  interpretation of kernel methods on graphs based on graphical
  models.  Motivated by the challenges of selecting the bandwidth
  parameter in SPoG estimators or the kernel map in kernel-based
  methods, the present paper further proposes two multi-kernel
  approaches with complementary strengths. Whereas the first enables
  estimation of the unknown bandwidth of bandlimited signals, the
  second allows for efficient graph filter selection.  Numerical tests
  with synthetic as well as real data demonstrate the merits of the
  proposed methods relative to state-of-the-art alternatives.

\end{abstract}
\begin{keywords}
  Graph signal reconstruction, kernel regression, multi-kernel
  learning.
\end{keywords}

\section{Introduction}
\label{sec:intro}

\cmt{Problem motivation}

Graph data play a central role in analysis and inference tasks for
social, brain, communication, biological, transportation, and sensor
networks~\cite{kolaczyck2009}, thanks to their ability to capture
relational information.  Vertex attributes or features associated with
vertices can be interpreted as functions or signals defined on graphs.
\cmt{example} In social networks for instance, where a vertex
represents a person and an edge corresponds to a friendship relation,
such a function may denote e.g. the person's age, location, or rating
of a given movie.

\cmt{Problem overview} Research efforts over the last years are
centered on estimating or processing functions on graphs; see
e.g. \cite{kondor2002diffusion,smola2003kernels,kolaczyck2009,shuman2013emerging,sandryhaila2013discrete,chapelle2006}.
\cmt{parsimony} Existing approaches rely on the premise that signals
obey a certain form of parsimony relative to the graph topology.
\cmt{example} For instance, it seems reasonable to estimate a
person's age by looking at their friends' age.  The present paper
deals with a general version of this task, where the goal is to
estimate a graph signal given noisy observations on a subset of
vertices.

\cmt{Literature} 

\cmt{Machine learning} The machine learning community has already
looked at SPoG-related issues in the context of \emph{semi-supervised
  learning} under the term of \cmt{name}\emph{transductive} regression
and
classification~\cite{chapelle2006,chapelle1999transductive,cortes2007transductive}.
\cmt{rely on}Existing approaches rely on smoothness assumptions
\cmt{tools} for inference of processes over graphs using
\emph{nonparametric}
methods~\cite{chapelle2006,kondor2002diffusion,smola2003kernels,belkin2006manifold}.
\cmt{mainly classification} Whereas some works consider estimation of
real-valued
signals~\cite{chapelle1999transductive,belkin2006manifold,cortes2007transductive,lafferty2007regression},
most in this body of literature have focused on estimating
binary-valued functions; see e.g.~\cite{chapelle2006}.  \cmt{SPoG}On
the other hand, function estimation has also been investigated
recently by the community of signal processing on graphs (SPoG)
\cmt{name}under the term \emph{signal
  reconstruction}~\cite{narang2013structured,narang2013localized,gadde2015probabilistic,tsitsvero2016uncertainty,chen2015theory,anis2016proxies,wang2015local,marques2015aggregations}.
\cmt{tools} Existing approaches commonly adopt \emph{parametric}
estimation tools \cmt{rely on}and rely on \emph{bandlimitedness}, by
which the signal of interest is assumed to lie in the span of the
$\blnum$ leading eigenvectors of the graph Laplacian or the adjacency
matrix~\cite{segarra2015percolation,tsitsvero2016uncertainty,anis2016proxies,narang2013localized,gadde2015probabilistic,wang2015local,marques2015aggregations}.
\cmt{mainly estimation} Different from machine learning works, 
SPoG research is mainly concerned with estimating real-valued
functions.

\cmt{Goal of paper} The present paper cross-pollinates ideas and
broadens both machine learning and SPoG perspectives under the
\emph{unifying} framework of kernel-based learning. \cmt{C1: kernel
  methods in SPoG}The first part unveils the implications of adopting
this standpoint and demonstrates how it naturally accommodates a
number of SPoG concepts and tools.  \cmt{benefits}From a high level,
this connection \cmt{bounds}(i) brings to bear performance bounds and
algorithms from transductive regression~\cite{cortes2007transductive}
and the extensively analyzed general kernel methods (see
e.g.~\cite{scholkopf2002}); \cmt{representer}(ii) offers the
possibility of reducing the dimension of the optimization problems
involved in Tikhonov regularized estimators by invoking the so-called
\emph{representer theorem}~\cite{kimeldorf1971spline}; and,
\cmt{covariance kernels}(iii) it provides guidelines for
systematically selecting parameters in existing signal reconstruction
approaches by leveraging the connection with linear minimum
mean-square error (LMMSE) estimation via \emph{covariance kernels}.

\cmt{kernel methods on graphs} Further implications of applying kernel
methods to graph signal reconstruction are also explored.  \cmt{novel
  proof}Specifically, it is shown that the finite dimension of graph
signal spaces allows for an insightful proof of the representer
theorem which, different from existing proofs relying on functional
analysis, solely involves linear algebra arguments.  \cmt{graphical
  models}Moreover, an intuitive probabilistic interpretation of graph
kernel methods is introduced based on graphical models. These findings
are complemented with a technique to deploy regression with Laplacian
kernels in big-data setups.

\cmt{generalizes} It is further established that a number of existing
signal reconstruction approaches, including the least-squares (LS)
estimators for bandlimited signals
from~\cite{narang2013localized,gadde2015probabilistic,tsitsvero2016uncertainty,chen2015theory,anis2016proxies,narang2013structured};
the Tikhonov regularized estimators
from~\cite{narang2013localized,belkin2004regularization,shuman2013emerging}
and~\cite[eq. (27)]{chen2015recovery}; and the maximum a posteriori
estimator in \cite{gadde2015probabilistic}, can be viewed as kernel
methods on \emph{reproducing kernel Hilbert spaces} (RKHSs) of graph
signals.  \cmt{notions}Popular notions in SPoG such as graph filters,
the graph Fourier transform, and bandlimited signals can also be
accommodated under the kernel framework.  \cmt{graph filter} First, it
is seen that a \emph{graph filter}~\cite{shuman2013emerging} is
essentially a kernel smoother~\cite{wahba1990splinemodels}.
\cmt{bandlimited signals}Second, \emph{bandlimited kernels} are
introduced to accommodate estimation of bandlimited signals.
\cmt{graph Fourier transform} Third, the connection between the
so-called \emph{graph Fourier transform}~\cite{shuman2013emerging}
(see~\cite{chen2015theory,sandryhaila2013discrete} for a related
definition) and Laplacian
kernels~\cite{kondor2002diffusion,smola2003kernels} is delineated.
Relative to methods relying on the bandlimited property (see
e.g.~\cite{narang2013localized,gadde2015probabilistic,tsitsvero2016uncertainty,chen2015theory,anis2016proxies,narang2013structured,wang2015local}),
kernel methods offer increased flexibility in leveraging prior
information about the graph Fourier transform of the estimated signal.

\cmt{C2: MKL} The second part of the paper pertains to the challenge
of model selection.  \cmt{Motivation: }\cmt{SPoG}On the one hand, a
number of reconstruction schemes in
SPoG~\cite{narang2013localized,gadde2015probabilistic,tsitsvero2016uncertainty,chen2015theory,wang2015local}
require knowledge of the signal bandwidth, which is typically
unknown~\cite{narang2013structured,anis2016proxies}. Existing
approaches for determining this bandwidth rely solely on the set of
sampled vertices, disregarding the
observations~\cite{narang2013structured,anis2016proxies}.
\cmt{Machine learning}On the other hand, existing kernel-based
approaches~\cite[Ch. 8]{kolaczyck2009} necessitate proper kernel
selection, which is computationally inefficient through
cross-validation.
  
\cmt{Goal} The present paper addresses both issues by means of two
multi-kernel learning (MKL) techniques having complementary strengths.
\cmt{literature} Heed existing MKL methods on graphs are confined to
estimating binary-valued
signals~\cite{zhu2004nonparametric,lanckriet2004semidefinite,cristianini2001alignment}. This
paper on the other hand, is concerned with MKL algorithms for
real-valued graph signal reconstruction.  \cmt{algorithms}The novel
graph MKL algorithms optimally combine the kernels in a given
dictionary and simultaneously estimate the graph signal by solving a
single optimization problem.
  
    

\cmt{Paper structure} The rest of the paper is structured as
follows. Sec.~\ref{sec:statement} formulates the problem of graph
signal reconstruction.  Sec.~\ref{sec:kernellearning} presents
kernel-based learning as an encompassing framework for graph signal
reconstruction, and explores the implications of adopting such a
standpoint.  Two MKL algorithms are then presented in
Sec.~\ref{sec:mkl}.  Sec.~\ref{sec:numericaltests} complements
analytical findings with numerical tests by comparing with competing
alternatives via synthetic- and real-data experiments. Finally, 
concluding remarks are highlighted in Sec.~\ref{sec:conclusions}.

\cmt{notation}
\noindent {\bf Notation.}  $(\cdot)_\vertexnum$ denotes the remainder
of integer division by $\vertexnum$; $\delta[\cdot]$ the Kronecker
delta, and $\indicator[C]$ the indicator of condition $C$, returning 1
if $C$ is satisfied and 0 otherwise.  Scalars are denoted by lowercase
letters, vectors by bold lowercase, and matrices by bold uppercase.
The $(i,j)$th entry of matrix $\bm A$ is $(\bm A)_{i,j}$. Notation
$||\cdot ||_2$ and $\tr{\cdot}$ respectively represent Euclidean norm
and trace; $\bm I_N$ denotes the $N \times N$ identity matrix; $\bm
e_i$ is the $i$-th canonical vector of $\mathbb{R}^M$, while $\bm 0$
($\bm 1$) is a vector of appropriate dimension with all (ones). The
span of the columns of $\bm A$ is denoted by $\columnspan\{\bm A\}$,
whereas $\bm A\succ \bm B$ (resp. $\bm A\succeq \bm B$) means that
$\bm A-\bm B$ is positive definite (resp. semi-definite).
Superscripts $^T$ and $~^\dagger$ respectively stand for transposition
and pseudo-inverse, whereas $\expectednb$ denotes expectation.

\section{Problem Statement}
\label{sec:statement}
\cmt{Definitions}

\cmt{graph} A graph is a tuple $\mathcal{G}:=(\vertexset,w)$, where
$\vertexset:=\{v_1, \ldots, v_N\}$ is the vertex set, and $w:
\vertexset \times \vertexset \rightarrow [0,+\infty)$ is a map
assigning a weight to each vertex pair. For simplicity, it is assumed
that $w(v,v)=0~\forall v\in \vertexset$.  This paper focuses on
\emph{undirected} graphs, for which $w(v,v') = w(v',v)~\forall v,v'\in
\vertexset$.  A graph is said to be \emph{unweighted} if $w(v,v')$ is
either 0 or 1.  \cmt{topology} The edge set $\edgeset $ is the support
of $w$, i.e., $\edgeset :=\{(v,v')\in \vertexset \times \vertexset:
w(v,v')\neq 0\}$.  Two vertices $v$ and $v'$ are \emph{adjacent},
\emph{connected}, or \emph{neighbors} if $(v,v')\in \edgeset$.  The
$n$-th neighborhood $\mathcal{N}_\vertexind$ is the set of neighbors
of $v_\vertexind$, i.e., $\mathcal{N}_\vertexind\define\{ v \in
\vertexset:(v,v_\vertexind)\in \edgeset \}$.  \cmt{algebraic defs} The
information in $w$ is compactly represented by the $N\times N$
weighted adjacency matrix $\adjacencymat$, whose
$(\vertexind,\vertexindp)$-th entry is
$w(v_\vertexind,v_\vertexindp)$; the $\vertexnum \times \vertexnum$
diagonal \emph{degree} matrix $\bm{D}$, whose
$(\vertexind,\vertexind)$-th entry is $\sum_{\vertexindp=1}^N
w(v_\vertexind,v_\vertexindp)$; and the \emph{Laplacian} matrix
$\bm{L}:=\bm{D}-\adjacencymat$, which is symmetric and positive
semidefinite~\cite[Ch. 2]{kolaczyck2009}. The latter is sometimes
replaced with its normalized version
$\bm{D}^{-1/2}\bm{L}\bm{D}^{-1/2}$, whose eigenvalues are confined to
the interval $[0,2]$.

\cmt{function} A real-valued function (or signal) on a graph is a map
$ \signalfun:\mathcal{V} \rightarrow \mathbb{R}$. As mentioned in
Sec.~\ref{sec:intro}, the value $\signalfun(v)$ represents an
attribute or feature of $v\in \vertexset$, such as age, political
alignment, or annual income of a person in a social network. Signal
$\signalfun$ is thus represented by
$\signalvec:=[\signalfun(v_1),\ldots,\signalfun(v_N)]^T$.

\cmt{Problem}

\cmt{observations} Suppose that a collection of noisy samples (or
observations) $y_s = \signalfun(v_{\vertexind_s}) + \noisesamp_s,
\quad s=1,\ldots,S$, is available, where $\noisesamp_s$ models noise
and $\mathcal{S}:=\{\vertexind_1,\ldots,\vertexind_S\}$ contains the
indices $1\leq \vertexind_1<\cdots<\vertexind_S\leq \vertexnum$ of the
sampled vertices. In a social network, this may be the case if a
subset of persons have been surveyed about the attribute of interest
(e.g. political alignment).  \cmt{goal}Given
$\{(\vertexind_\sampleind,y_\sampleind)\}_{\sampleind=1}^\samplenum$,
and assuming knowledge of $\mathcal{G}$, the goal is to estimate
$\signalfun$. This will provide estimates of $\signalfun(v)$ both at
observed and unobserved vertices $v\in \vertexset$.
\cmt{vector-matrix form} By defining $\observationvec \define
[y_{1},\ldots,y_{S}]^T$, the observation model is summarized as
\begin{equation}
\label{eq:observationsvec}
\observationvec =\samplingmat \signalvec  + \noisevec
\end{equation}
where $\noisevec \define [\noisesamp_{1},\ldots,\noisesamp_{S}]^T$ and $\samplingmat$ is
an $S\times N$ matrix with entries $(s,\vertexind_s)$,
$s=1,\ldots,S$, set to one, and the rest set to zero.

\section{Unifying the reconstruction of graph signals}
\label{sec:kernellearning}

\cmt{Introduction to kernel methods} Kernel methods constitute the
``workhorse'' of statistical learning for nonlinear function
estimation~\cite{scholkopf2002}.  Their popularity can be ascribed to
their simplicity, flexibility, and good performance.  \cmt{Section
  overview} This section presents kernel regression as a novel
unifying framework for graph signal reconstruction.

\cmt{RKHSs}

   \cmt{goal} Kernel regression seeks an estimate of
  $\signalfun$ in an RKHS $\rkhs$, which is the space of functions
  $\rkhsfun:\vertexset \rightarrow \rfield$ defined as
\begin{align}
\label{eq:rkhsdef}
\rkhs:=\left\{\rkhsfun:\rkhsfun(v) =
  \sum_{\vertexind=1}^{\vertexnum}\fullalpha_\vertexind
  \kernelmap(v,v_\vertexind),~\fullalpha_\vertexind\in \rfield \right\}.
\end{align}
\cmt{kernel} The \emph{kernel map} $\kernelmap:\vertexset\times
\vertexset\rightarrow \rfield$ is any function defining a symmetric
and positive semidefinite $N\times N$  matrix with entries
$[\fullkernelmat]_{\vertexind,\vertexindp} \define
\kernelmap(v_\vertexind,v_\vertexindp)~\forall
\vertexind,\vertexindp$~\cite{scholkopf2001representer}.  \cmt{kernel
  intuition}Intuitively, $\kernelmap(v,v')$ is a basis function in
\eqref{eq:rkhsdef} measuring similarity between the values of
$\signalfun$ at $v$ and $v'$.  For instance, if a \emph{feature
  vector} $\bm x_\vertexind\in \rfield^\featurevecdim$ containing
attributes of the entity represented by $v_\vertexind$ is known for
$\vertexind=1,\ldots,\vertexnum$, one can employ the popular
\emph{Gaussian kernel} $\kernelmap(v_\vertexind,v_\vertexindp)=\exp\{
-||\bm x_\vertexind-\bm x_\vertexindp||^2/\sigma^2\}$, where
$\sigma^2>0$ is a user-selected parameter~\cite{scholkopf2002}. When
such feature vectors $\bm x_\vertexind$ are not available, the graph
topology can be leveraged to construct graph kernels as detailed in
Sec.~\ref{sec:graphkernels}.

Different from RKHSs of functions $f(\bm x)$ defined over infinite
sets, the expansion in \eqref{eq:rkhsdef} is finite since $\vertexset$
is finite. This implies that RKHSs of graph signals are
finite-dimensional spaces.  \cmt{General form $\rkhsfun$ in graph RKHS
}From \eqref{eq:rkhsdef}, it follows that any signal in $\rkhs$ can be
expressed as:
\begin{align}
\label{eq:generalform}
\rkhsvec \define [\rkhsfun(v_1),\ldots,\rkhsfun(v_N)]^T = \fullkernelmat  \fullalphavec
\end{align}
for some $N\times 1$ vector $\fullalphavec\define[\fullalpha_1,\ldots,\fullalpha_\vertexnum]^T$.
 \cmt{Inner product}Given two functions $\rkhsfun(v) \define
  \sum_{\vertexind=1}^{\vertexnum}\fullalpha_\vertexind
  \kernelmap(v,v_\vertexind)$ and $\rkhsfun'(v) \define
  \sum_{\vertexind=1}^{\vertexnum}\fullalpha'_\vertexind
  \kernelmap(v,v_\vertexind)$, their RKHS inner product is defined
  as\footnote{Whereas $\rkhsfun$ denotes a \emph{function}, symbol
    $\rkhsfun(v)$ represents the \emph{scalar} resulting from
    evaluating $\rkhsfun$ at vertex  $v$. }
\begin{align}
\label{eq:innerproduct}
\langle \rkhsfun,\rkhsfun'\rangle_\mathcal{H} := \sum_{\vertexind=1}^{N}\sum_{\vertexindp=1}^{N}\fullalpha_\vertexind \fullalpha'_\vertexindp \kernelmap(
  v_\vertexind,v_\vertexindp) = \fullalphavec^T \fullkernelmat \fullalphavec'
\end{align}
where
$\fullalphavec'\define[\fullalpha'_1,\ldots,\fullalpha'_\vertexnum]^T$. 
\cmt{norm}The RKHS norm is defined by
\begin{align}
\label{eq:norm}
||\rkhsfun||_\mathcal{H}^2\define \langle \rkhsfun,\rkhsfun \rangle_\rkhs=
\fullalphavec^T \fullkernelmat \fullalphavec
\end{align}
and will be used as a regularizer to control overfitting.
\cmt{Intuition}As a special case, setting $\fullkernelmat=\bm I_N$
recovers the standard inner product $\langle
\rkhsfun,\rkhsfun'\rangle_\rkhs = \rkhsvec^T \rkhsvec'$, and Euclidean
norm $||\rkhsfun||_\rkhs^2 = ||\rkhsvec||_2^2$. Note that when
$\fullkernelmat\succ \bm 0$, the set of functions of the form
\eqref{eq:generalform} equals $\rfield^\vertexnum$. Thus, two RKHSs
with strictly positive definite kernel matrices contain the same
functions. They differ only in their RKHS inner products and
norms. Interestingly, this observation establishes that any positive
definite kernel is \emph{universal}~\cite{carmeli2010vector} for graph
signal reconstruction.

\cmt{reproducing property} The term \emph{reproducing kernel} stems
from the reproducing property. Let
$\kernelmap(\cdot,v_{\vertexind_0})$ denote the map $v
\mapsto \kernelmap(v,v_{\vertexind_0})$, where $\vertexind_0\in
\{1,\ldots,N\}$. Using \eqref{eq:innerproduct}, the reproducing
property can be expressed as $\langle
\kernelmap(\cdot,v_{\vertexind_0}),\kernelmap(\cdot,v_{\vertexind_0'})\rangle_\mathcal{H}
= \bm e_{\vertexind_0}^T \fullkernelmat \bm e_{\vertexind_0'} =
\kernelmap(v_{\vertexind_0},v_{\vertexind_0'})$. Due to the linearity
of inner products and the fact that all signals in $\rkhs$ are the
superposition of functions of the form
$\kernelmap(\cdot,v_{\vertexind})$, the reproducing property asserts
that inner products can be obtained just by evaluating
$\kernelmap$. The reproducing property is of paramount importance when
dealing with an  RKHS of functions defined on \emph{infinite}
spaces (thus excluding RKHSs of graph signals), since it offers an
efficient alternative to the costly multidimensional integration
required by inner products such as $\langle f_1,f_2\rangle_{L^2} :=
\int_\mathcal{X} f_1(\bm x)f_2(\bm x)d\bm x$.

\cmt{criterion}

\cmt{def} Given $\{y_\sampleind\}_{\sampleind=1}^\samplenum$,
RKHS-based function estimators are obtained by solving functional
minimization problems formulated as
\begin{equation}
	\signalestfun := \argmin_{\rkhsfun\in \mathcal{H}} \frac{1}{S} \sum_{s=1}^S
        \mathcal{L} (v_{\vertexind_s}, y_s, \rkhsfun(v_{\vertexind_s}))
+ \regpar \Omega(||\rkhsfun||_\mathcal{H})
	\label{eq:single-general}
\end{equation}
where the regularization parameter $\regpar>0$ controls overfitting,
the increasing function $\Omega$ is used to promote smoothness, and
the loss function $\mathcal{L}$ measures how estimates deviate from
the data. The so-called \emph{square loss}
$\mathcal{L}(v_{\vertexind_s},y_s,\rkhsfun(
v_{\vertexind_s}))\define\left[y_s-\rkhsfun(v_{\vertexind_s})\right]^2$
constitutes a popular choice for $\mathcal{L}$, whereas $\Omega$ is
often set to $\Omega(\zeta) = |\zeta|$ or $ \Omega(\zeta) = \zeta^2$.

\cmt{simplify notation} To simplify notation, consider loss functions
expressible as $\lossfun(v_{\vertexind_s},y_s,\rkhsfun(
v_{\vertexind_s}))=\lossfun(y_s-\rkhsfun(v_{\vertexind_s}))$;
extensions to more general cases are straightforward. The
vector-version of such a function is $ \mathcal{L} (\bm y -
\samplingmat \rkhsvec) := \sum_{s=1}^S
\lossfun(y_s-\rkhsvec(v_{\vertexind_s}))$.  \cmt{solution}Substituting
\eqref{eq:generalform} and \eqref{eq:norm} into
\eqref{eq:single-general} shows that $\signalestvec$ can be obtained
as $\signalestvec = \fullkernelmat\fullalphaestvec$, where
\begin{equation}
\fullalphaestvec := \argmin_{\fullalphavec\in \rfield^\vertexnum} \frac{1}{S} 
        \mathcal{L} (\bm  y - \samplingmat \fullkernelmat \fullalphavec)
+ \regpar \Omega((\fullalphavec^T \fullkernelmat \fullalphavec)^{1/2}).
	\label{eq:solutiongraphs}
\end{equation}
\cmt{Alternative form} An alternative form of
\eqref{eq:solutiongraphs} that will be frequently used in the sequel
results upon noting that $\fullalphavec^T \fullkernelmat \fullalphavec
= \fullalphavec^T \fullkernelmat \fullkernelmat^\dagger \fullkernelmat
\fullalphavec = \rkhsvec^T \fullkernelmat^\dagger \rkhsvec$. Thus, one
can rewrite \eqref{eq:solutiongraphs} as
\begin{equation}
\signalestvec := \argmin_{\rkhsvec\in \columnspan\{\fullkernelmat\}} \frac{1}{S} 
        \mathcal{L} (\bm  y - \samplingmat\rkhsvec)
+ \regpar \Omega((\rkhsvec^T \fullkernelmat^\dagger \rkhsvec)^{1/2}).
	\label{eq:solutiongraphsf}
\end{equation}
If $\fullkernelmat\succ \bm 0$, the constraint $\rkhsvec\in
\columnspan\{\fullkernelmat\}$ can be omitted, and $
\fullkernelmat^\dagger $ can be replaced with $\fullkernelmat\inv$. If
$\fullkernelmat$ contains null eigenvalues, it is customary to remove
the constraint by replacing $\fullkernelmat $ (or
$\fullkernelmat^\dagger $) with a perturbed version
$\fullkernelmat+\epsilon \bm I $ (respectively
$\fullkernelmat^\dagger+\epsilon \bm I $), where $\epsilon>0$ is a
small constant.  Expression \eqref{eq:solutiongraphsf} shows that
kernel regression unifies and subsumes the Tikhonov-regularized graph
signal reconstruction schemes
in~\cite{narang2013localized,belkin2004regularization,shuman2013emerging}
and~\cite[eq. (27)]{chen2015recovery} by properly selecting
$\fullkernelmat$, $\mathcal{L}$, and $\Omega$ (see
Sec.~\ref{sec:graphkernels}).

\subsection{Representer theorem}

\cmt{Motivation representer theorem} Although graph signals can be
reconstructed from \eqref{eq:solutiongraphs}, such an approach
involves optimizing over $N$ variables. This section shows that a
solution can be obtained by solving an optimization problem in $S$
variables, where typically $S\ll N$.

\cmt{infinite case} The representer
theorem~\cite{kimeldorf1971spline,scholkopf2001representer} plays an
instrumental role in the non-graph setting of infinite-dimensional
$\mathcal{H}$, where \eqref{eq:single-general} cannot be directly
solved. This theorem enables a solver by providing a finite
parameterization of the function $ \signalestfun$ in
\eqref{eq:single-general}.  \cmt{finite case}On the other hand, when
$\rkhs$ comprises graph signals,~\eqref{eq:single-general} is
inherently finite-dimensional and can be solved directly. However, the
representer theorem can still be beneficial to reduce the dimension of
the optimization in \eqref{eq:solutiongraphs}.

 \cmt{representer theorem} 
\begin{mytheorem}[Representer theorem]
  \thlabel{prop:representer} The solution to the functional
  minimization in \eqref{eq:single-general} can be expressed as
\begin{align}
\label{eq:representer}
\signalestfun(v) = \sum_{s=1}^{S} \samplealpha_s \kernelmap(v ,v_{\vertexind_s}) 
\end{align}
for some $\samplealpha_s \in \rfield$, $s=1,\ldots,S$.
\end{mytheorem}

\cmt{proof} The conventional proof for the representer theorem
involves tools from functional
analysis~\cite{scholkopf2001representer}. However, when $\rkhs$
comprises functions defined on finite spaces, such us graph signals,
an insightful proof can be obtained relying solely on linear algebra
arguments (see Appendix~\ref{appendix:representer}).

\cmt{interpretation} Since the solution $\signalestfun$ of
\eqref{eq:single-general} lies in $ \rkhs$, it can be expressed as
$\signalestfun = \sum_{\vertexind=1}^{N}\fullalpha_\vertexind
\kernelmap(v,v_\vertexind)$ for some
$\{\fullalpha_\vertexind\}_{\vertexind =1}^\vertexnum$.
\thref{prop:representer} states that the terms corresponding to
unobserved vertices $v_\vertexind$, $\vertexind \notin \samplingset$,
play no role in the kernel expansion of the estimate; that is,
$\fullalpha_\vertexind =0,~\forall\vertexind\notin \mathcal{S}$. Thus,
whereas \eqref{eq:solutiongraphs} requires optimization over
$\vertexnum$ variables, \thref{prop:representer} establishes that a
solution can be found by solving a problem in $\samplenum$ variables,
where typically $\samplenum\ll \vertexnum$. Clearly, this conclusion
carries over to the signal reconstruction schemes
in~\cite{narang2013localized,belkin2004regularization,shuman2013emerging}
and~\cite[eq. (27)]{chen2015recovery}, since they constitute special
instances of kernel regression.  \cmt{nonparametric name}The fact that
the number of parameters to be estimated after applying
\thref{prop:representer} depends on (in fact, equals) the number of
samples $\samplenum$ justifies why $ \signalestfun$ in
\eqref{eq:single-general} is referred to as a \emph{nonparametric
  estimate}.

\cmt{optimal alphas} \thref{prop:representer} shows the form of
$\signalestfun$ but does not provide the optimal
$\{\samplealpha_s\}_{s=1}^S$, which is found after substituting
\eqref{eq:representer} into \eqref{eq:single-general} and solving the
resulting optimization problem with respect to these coefficients. To
this end, let
$\samplealphavec:=[\samplealpha_1,\ldots,\samplealpha_\samplenum]^T$,
 and write  $\fullalphavec = \samplingmat^T \samplealphavec $
to deduce that 
\begin{align}
\label{eq:graphestimate}
\signalestvec= \fullkernelmat \fullalphavec
= \fullkernelmat \samplingmat^T \samplealphavec.
\end{align}
From \eqref{eq:solutiongraphs} and \eqref{eq:graphestimate}, the
optimal $\samplealphavec$ can be found as
\begin{equation}
\samplealphaestvec := \argmin_{\samplealphavec\in \rfield^\samplenum} \frac{1}{S} 
        \mathcal{L} (\bm  y - \samplekernelmat \samplealphavec)
+ \regpar \Omega((\samplealphavec^T \samplekernelmat \samplealphavec)^{1/2})
	\label{eq:solutiongraphsmatrixform}
\end{equation}
where $\samplekernelmat\define \samplingmat \fullkernelmat \samplingmat^T$.

\numberedexample{ex:ridgeregression} \emph{(kernel ridge regression)}.
\cmt{def} For $\mathcal{L}$ chosen as the square loss and
$\Omega(\zeta) =\zeta^2$, the $\signalestfun$ in
\eqref{eq:single-general} is referred to as the \emph{kernel ridge
  regression} estimate. It is given by $\signalridgeestvec =
\fullkernelmat \samplealphaestvec$, where
\begin{subequations}
	\label{eq:single-lq-m}
\begin{align}
	\label{eq:single-lq-m-a}
	\samplealphaestvec &:= \argmin_{\samplealphavec\in \rfield^\samplenum} \frac{1}{S} \left\| \bm{y} -
\samplekernelmat\bm{\samplealphavec} \right\|^2 + \regpar \bm{\samplealphavec}^T \samplekernelmat
\bm{\samplealphavec}\\
&= (\samplekernelmat + \regpar S \bm{I}_S )^{-1} \bm{y} .
\end{align}
\end{subequations}
Therefore, $\signalridgeestvec$ can be expressed as
\begin{align}
\label{eq:ridgeregressionestimate}
\signalridgeestvec
= \fullkernelmat \samplingmat^T (\samplekernelmat + \regpar S \bm{I}_S )^{-1} \bm{y}.
\end{align}
\cmt{KRR generalizes SPoG estimators} As seen in the next section,
\eqref{eq:ridgeregressionestimate} generalizes a number of existing
signal reconstructors upon properly selecting $\fullkernelmat$. Thus,
\thref{prop:representer} can also be used to simplify
Tikhonov-regularized estimators such as the one
in~\cite[eq. (15)]{narang2013localized}. To see this, just note that
\eqref{eq:ridgeregressionestimate} inverts an $\samplenum \times
\samplenum$ matrix whereas~\cite[eq. (16)]{narang2013localized}
entails the inversion of an $\vertexnum \times \vertexnum$ matrix.

\numberedexample{ex:svr} \emph{(support vector regression)}.  If
$\mathcal{L}$ equals the so-called $\epsilon$-insensitive loss
$\mathcal{L}(v_{\vertexind_s},y_s,\rkhsfun( v_{\vertexind_s}))\define
\max(0,|y_s-\rkhsfun( v_{\vertexind_s})|-\epsilon)$ and $\Omega(\zeta)
=\zeta^2$, then \eqref{eq:single-general} constitutes a support vector
machine for regression~(see e.g.~\cite[Ch. 1]{scholkopf2002}).

\subsection{Graph kernels for signal reconstruction}
\label{sec:graphkernels}

\cmt{Motivation} When estimating functions on graphs, conventional
kernels such as the aforementioned Gaussian kernel cannot be applied
because the underlying set where graph signals are defined is not a
metric space. Indeed, no vertex addition $v_{\vertexind} +
v_{\vertexind'}$, scaling $\beta v_{\vertexind}$, or norm
$||v_{\vertexind}||$ can be naturally defined on $\vertexset$.
\cmt{feature embedding} An alternative is to embed $\vertexset$ into
an Euclidean space via a feature map $\phi:\vertexset \rightarrow
\rfield^\featurevecdim$, and apply a conventional kernel afterwards.
However, for a given graph it is generally unclear how to design such
a map or select $\featurevecdim$, which motivates the adoption of
graph kernels~\cite{smola2003kernels}.  \cmt{Overview}The rest of this
section elaborates on three classes of graph kernels, namely
\emph{Laplacian}, \emph{bandlimited}, and novel \emph{covariance}
kernels for reconstructing  graph signals.

\subsubsection{Laplacian kernels}
\label{sec:laplaciankernels}

\cmt{overview} The term Laplacian kernel comprises a wide family of
kernels obtained by applying a certain function to the Laplacian matrix
$\bm L$. From a theoretical perspective, Laplacian kernels are well
motivated since they constitute the graph counterpart of the so-called
\emph{translation invariant kernels} in Euclidean
spaces~\cite{smola2003kernels}.  \cmt{this section}This section
reviews Laplacian kernels, provides novel insights in terms of
interpolating signals, and highlights their versatility in capturing
prior information about the \emph{graph Fourier transform} of the
estimated signal.

   \cmt{Definition} Let $0=\lambda_1\leq \lambda_2\leq \ldots\leq
    \lambda_N$ denote the eigenvalues of the graph Laplacian matrix
    $\bm L$, and consider the eigendecomposition $\bm
    L=\laplacianevecmat \bm \Lambda \laplacianevecmat^T$, where $\bm
    \Lambda\define\diag{\lambda_1,\ldots,\lambda_N}$. A Laplacian kernel is
    a kernel map $\kernelmap$ generating a matrix $\fullkernelmat$ of the
    form
\begin{align}
  \fullkernelmat\define r^\dagger(\bm L) \define \laplacianevecmat
  r^\dagger(\bm \Lambda) \laplacianevecmat^T
\label{eq:laplacian_kernel}
\end{align}
where $r(\bm \Lambda)$ is the result of applying the user-selected
non-negative map $r:\rfield\rightarrow \rfield_+$ to the diagonal
entries of $\bm \Lambda$.  \cmt{Selection of r} For reasons that will
become clear, the map $r(\lambda)$ is typically increasing in
$\lambda$. Common choices include the diffusion kernel
$r(\lambda)=\exp\{\sigma^2\lambda/2\}$~\cite{kondor2002diffusion}, and
the $p$-step random walk kernel $r(\lambda) =
(a-\lambda)^{-p}$,~$a\geq 2$~\cite{smola2003kernels}.  Laplacian
regularization~\cite{smola2003kernels,zhou2004regularization,belkin2006manifold,forero2014dictionary,shuman2013emerging}
is effected by setting $r(\lambda)=1 + \sigma^2\lambda$ with
$\sigma^2$ sufficiently large.

\cmt{big data} Observe that obtaining $\fullkernelmat$ generally
requires an eigendecomposition of $\bm L$, which is computationally
challenging for large graphs ($\vertexnum\gg$). Two techniques to
reduce complexity in these \emph{big data} scenarios are proposed in
Appendix~\ref{sec:bigdata}.

At this point, it is prudent to offer interpretations and insights
into the principles behind the operation of Laplacian kernels.
\cmt{rewriting regularizer}Towards this objective, note first that the
regularizer from~\eqref{eq:solutiongraphsf} is an increasing function
of
\begin{align}
\rkhsvec^T  \fullkernelmat^\dagger   \rkhsvec
= \rkhsvec^T \laplacianevecmat r(\bm \Lambda)
\laplacianevecmat^T \rkhsvec
= \fourierrkhsvec^T  r(\bm \Lambda)
 \fourierrkhsvec =
 \sum_{\vertexind=1}^Nr(\lambda_\vertexind)|\fourierrkhsfun_\vertexind|^2
\label{eq:fourierregularizer}
\end{align}
where $\fourierrkhsvec \define \laplacianevecmat^T \rkhsvec\define
[\fourierrkhsfun_1,\ldots,\fourierrkhsfun_N]^T $ comprises the
projections of $\rkhsvec$ onto the eigenvectors of $\bm L$, and is
referred to as the \emph{graph Fourier transform} of $\rkhsvec$ in the
SPoG parlance~\cite{shuman2013emerging}.  \cmt{interpretation of
  $\fourierrkhsvec$}Before interpreting \eqref{eq:fourierregularizer},
it is worth elucidating the rationale behind this term.  Since
$\laplacianevecmat\define[\laplacianevec_1,\ldots,\laplacianevec_\vertexnum]$
is orthogonal, one can decompose $\rkhsvec$ as
 \begin{align}
 \label{eq:gft}
  \rkhsvec =
   \sum_{\vertexind=1}^\vertexnum \fourierrkhsfun_\vertexind
   \laplacianevec_\vertexind.
 \end{align}
 Because vectors $\{ \laplacianevec_\vertexind\}_{\vertexind=1}^N$, or
 more precisely their signal counterparts $\{
 \laplacianefun_\vertexind\}_{\vertexind=1}^N$, are
 \emph{eigensignals} of the so-called \emph{graph shift operator}
 $\laplacianevec \mapsto \bm L\laplacianevec$, \eqref{eq:gft}
 resembles the classical Fourier transform in the sense that it
 expresses a signal as a superposition of \emph{eigensignals} of a
 Laplacian operator~\cite{shuman2013emerging}.  \cmt{interpretation of
   frequency} Recalling from Sec.~\ref{sec:statement} that
 $w(v_\vertexind,v_{\vertexind'})$ denotes the weight of the edge
 between $v_\vertexind$ and $v_\vertexindp$, one can consider the
 smoothness measure for graph functions $\rkhsfun$ given by
\begin{align*}
\graphvariation \rkhsfun \define \frac{1}{2} \sum_{\vertexind =1}^N
\sum_{\vertexind'\in \mathcal{N}_\vertexind}
w(v_\vertexind,v_{\vertexind'})[ \rkhsfun(v_\vertexind) -
  \rkhsfun(v_{\vertexind'})]^2 = \rkhsvec^T \bm L \rkhsvec
\end{align*}
where the last equality follows from the definition of $\bm L\define
\bm D-\adjacencymat$. Clearly, it holds $\graphvariation
\laplacianefun_\vertexind = \lambda_\vertexind$. Since
$0=\lambda_1\leq \lambda_2\leq \ldots\leq \lambda_N$, it follows that
$0=\graphvariation \laplacianefun_1 \leq \ldots\leq \graphvariation
\laplacianefun_\vertexnum$. In analogy to signal processing for time
signals, where lower frequencies correspond to smoother eigensignals,
the index $\vertexind$, or alternatively the eigenvalue
$\lambda_\vertexind$, is interpreted as the \emph{frequency} of
$\laplacianefun_\vertexind$.

\cmt{interpretation of regularizer} It follows from
\eqref{eq:fourierregularizer} that the regularizer
in~\eqref{eq:solutiongraphsf} strongly penalizes those
$\fourierrkhsfun_\vertexind$ for which the corresponding
$r(\lambda_\vertexind)$ is large, thus promoting a specific structure
in this frequency domain. Specifically, one prefers
$r(\lambda_\vertexind)$ to be large whenever $|\fourierrkhsfun_n|^2$
is small and vice versa. The fact that
$|\fourierrkhsfun_\vertexind|^2$ is expected to decrease with $\vertexind$
for smooth $\rkhsfun$, motivates the adoption of an increasing
$r$~\cite{smola2003kernels}. Observe that Laplacian kernels can
capture richer forms of prior information than the signal reconstructors
of bandlimited
signals in~\cite{narang2013localized,gadde2015probabilistic,tsitsvero2016uncertainty,chen2015theory,wang2015local,marques2015aggregations},
since the latter can solely capture the support of the Fourier
transform whereas the former can also leverage magnitude information.

\numberedexample{ex:circulargraphs} \emph{(circular graphs)}.
\cmt{Overview}This example capitalizes on \thref{prop:representer} to
present a novel SPoG-inspired intuitive interpretation of
nonparametric regression with Laplacian kernels.
 \cmt{Def circular graph} To do so, a closed-form expression
for the Laplacian kernel matrix of a circular graph (or
ring) will be derived.   \cmt{Motivation
  circular graphs}This class of graphs has been commonly employed in
the literature to illustrate connections between SPoG and signal
processing of time-domain signals~\cite{sandryhaila2013discrete}.

\cmt{Laplacian} Up to vertex relabeling, an unweighted circular graph
satisfies $\weightfun(v_\vertexind,v_{\vertexind'})=\delta[
(\vertexind- \vertexind')_\vertexnum - 1]+\delta[ (\vertexind'-
\vertexind)_\vertexnum - 1]$.  Therefore, its Laplacian matrix can be
written as $\bm L = 2\bm I_\vertexnum - \rotationmat -
\rotationmat^T$, where $\rotationmat$ is the rotation matrix resulting
from circularly shifting the columns of $\bm I_\vertexnum$ one
position to the right, i.e.,
$(\rotationmat)_{\vertexind,\vertexindp}\define
\delta[(\vertexindp-\vertexind)_\vertexnum -1]$.  \cmt{Laplacian
  eigendecomposition} Matrix $\bm L$ is \emph{circulant} since its
$\vertexind$-th row can be obtained by circularly shifting the
$(\vertexind-1)$-st row one position to the right. Hence, $\bm L$ can
be diagonalized by the standard Fourier matrix~\cite{gray}, meaning
\begin{align}
\label{eq:circularevd}
\bm L = \uslaplacianevecmat \uslaplacianevalmat \uslaplacianevecmat^H
\end{align}
where $(\uslaplacianevecmat)_{\freqind,\freqindp}\define
(1/\sqrt{\vertexnum})\exp\{ j2\pi(\freqind-1)(\freqindp-1)/N \} $ is
the unitary inverse discrete Fourier transform matrix and
$(\uslaplacianevalmat)_{\freqind,\freqindp}\define
2[1-\cos(2\pi(\freqind-1)/\vertexnum)]\delta[\freqind-\freqindp]$.
Matrices $\uslaplacianevecmat$ and $\uslaplacianevalmat$ replace
$\laplacianevecmat$ and $\laplacianevalmat$ since, for notational
brevity, the eigendecomposition \eqref{eq:circularevd} involves
complex-valued eigenvectors and the eigenvalues have not been sorted
in ascending order.

 \cmt{Laplacian kernel as IDFT of spectrum} From
\eqref{eq:laplacian_kernel}, a Laplacian kernel matrix is given by
$\fullkernelmat \define \uslaplacianevecmat
r^\dagger(\uslaplacianevalmat) \uslaplacianevecmat^H \define
\uslaplacianevecmat \diag{\diagvec} \uslaplacianevecmat^H$, where
$\diagvec \define[\diagent_0,\ldots, \diagent_{\vertexnum-1}]^T$
has entries $\diagent_\vertexind =r^\dagger( 2[1-\cos(2\pi
\vertexind/\vertexnum)])$.
 \cmt{Entries of $\fullkernelmat$}It can be easily seen that 
$(\fullkernelmat)_{\freqind,\freqindp} = \fourierdiagent_{\freqind-\freqindp}$,
where 
\begin{align*}
\fourierdiagent_{m} &\define {\rm IDFT}\{\diagent_\vertexind\}
\define \frac{1}{\vertexnum}\sum_{\vertexind=0}^{\vertexnum-1} 
\diagent_\vertexind e^{
j\frac{2\pi}{\vertexnum}m\vertexind
}.
\end{align*}
If $r\left(2[1-\cos(2\pi \vertexind/\vertexnum)]\right)>0~\forall \vertexind$,
one has that
\begin{align}
\label{eq:idftcircular}
\fourierdiagent_{m} 
&=
\frac{1}{\vertexnum}\sum_{\vertexind=0}^{\vertexnum-1} 
\frac{e^{j\frac{2\pi}{\vertexnum}m\vertexind}}{
r\left(2[1-\cos(2\pi \vertexind/\vertexnum)]\right)
}.
\end{align}

\cmt{representer theorem \ra interpolating signals} Recall that
\thref{prop:representer} dictates $\signalestvec=
\sum_{\sampleind\in\samplingset}\fullalpha_\sampleind
\fullkernelvec_\sampleind $, where $\fullkernelmat \define[
\fullkernelvec_1,\ldots, \fullkernelvec_\vertexnum ]$. Since
$(\fullkernelmat)_{m,m'} = \fourierdiagent_{m-m'}$ and because
$\fourierdiagent_m$ is periodic in $m$ with period $\vertexnum$, it
follows that the vectors
$\{\fullkernelvec_\vertexind\}_{\vertexind=1}^\vertexnum$ are all
circularly shifted versions of each other. Moreover, since
$\fullkernelmat$ is positive semidefinite, the largest entry of
$\fullkernelvec_\sampleind$ is precisely the $\sampleind$-th one,
which motivates interpreting $\fullkernelvec_\sampleind$ as an
interpolating signal centered at $\sampleind$, which in turn suggests
that the expression $\signalestvec=
\sum_{\sampleind\in\samplingset}\fullalpha_\sampleind
\fullkernelvec_\sampleind $ can be thought of as a reconstruction
equation. From this vantage point, signals
$\{\fullkernelvec_\sampleind\}_{\sampleind \in \samplingset}$ play an
analogous role to sinc functions in signal processing of time-domain
signals. Examples of these interpolating signals are depicted in
Fig.~\ref{fig:interpolating}.

\begin{figure}[t]
 \centering \includegraphics[width=.45\textwidth]{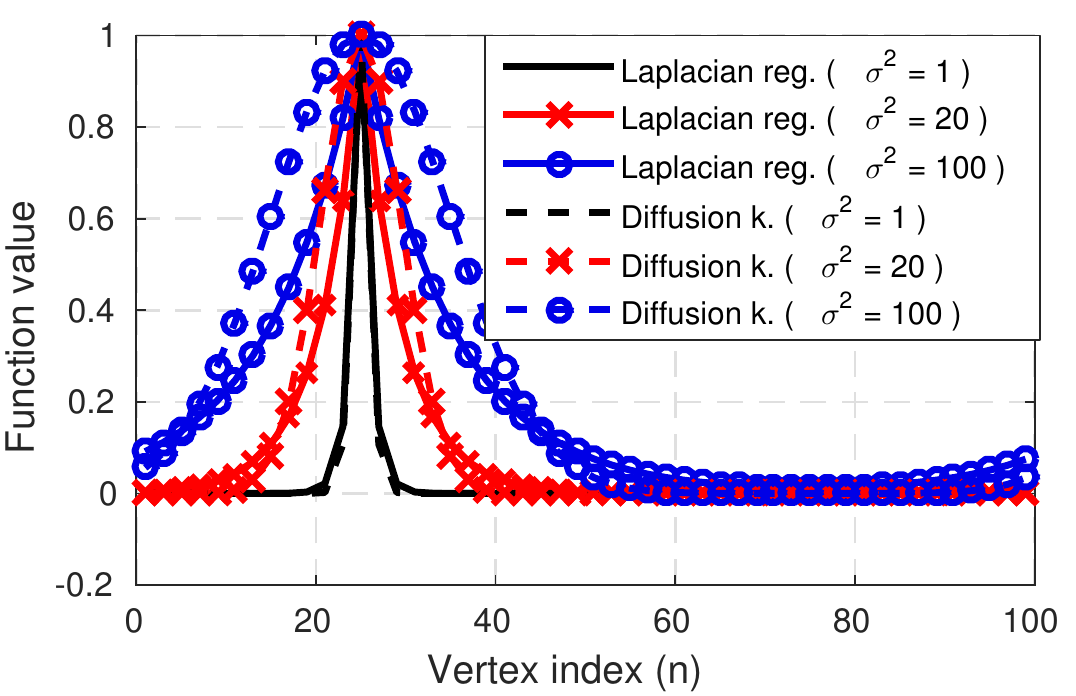}
 \caption{25-th column of $\fullkernelmat$ for a circular graph with
   $N = 100$ vertices. Different curves correspond to different
   parameter values for the Laplacian regularization and diffusion
   kernels.}
 \label{fig:interpolating}
\end{figure}

\subsubsection{Bandlimited kernels} 
\label{sec:bandlimited}

\cmt{Motivation} A number of signal reconstruction approaches in the
SPoG literature deal with graph bandlimited signals; see
e.g.~\cite{narang2013localized,gadde2015probabilistic,tsitsvero2016uncertainty,chen2015theory,wang2015local,anis2016proxies,marques2015aggregations,narang2013structured}.
\cmt{overview} Here, the notion of \emph{bandlimited kernel} is
introduced to formally show that the LS estimator for bandlimited
signals~\cite{narang2013localized,gadde2015probabilistic,tsitsvero2016uncertainty,chen2015theory,anis2016proxies,narang2013structured}
is a limiting case of the kernel ridge regression estimate
from~\eqref{eq:ridgeregressionestimate}. This notion will come handy
in Secs.~\ref{sec:mkl} and~\ref{sec:numericaltests} to estimate the
bandwidth of a bandlimited signal from the observations
$\{y_\sampleind\}_{\sampleind=1}^\samplenum$.

\cmt{def bandlimited} Signal $\rkhsfun$ is said to be
\emph{bandlimited} if it admits an expansion \eqref{eq:gft} with
$\fourierrkhsfun_n$ supported on a set $\blset \subset \{1,\ldots,
\vertexnum\}$; that is,
 \begin{align}
 \label{eq:defbandlimited}
  \rkhsvec =
   \sum_{\vertexind\in \blset} \fourierrkhsfun_\vertexind
   \laplacianevec_\vertexind   = \bllaplacianevecmat \blfouriervec
 \end{align}
 where $ \bllaplacianevecmat$ contains the columns of
 $\laplacianevecmat$ with indexes in $\blset$, and $\blfouriervec$ is a
 vector stacking $\{\fourierrkhsfun_\vertexind\}_{\vertexind\in
   \blset}$.  The \emph{bandwidth} of $\rkhsfun$ can be defined as the
 cardinality $\blnum\define|\blset|$, or, as the greatest element of
 $\blset$.

 \cmt{LS} If $\signalfun$ is bandlimited, it follows from
 \eqref{eq:observationsvec} that $ \observationvec=
 \samplingmat\signalvec + \noisevec=
 \samplingmat\bllaplacianevecmat\blfouriervec + \noisevec$ for some
 $\blfouriervec$. The LS estimate of $\signalvec$ is therefore given
 by~\cite{narang2013localized,gadde2015probabilistic,tsitsvero2016uncertainty,chen2015theory,anis2016proxies,narang2013structured}
\begin{subequations}
\label{eq:blestimate}
\begin{align}
  \signallsestvec &= \bllaplacianevecmat \argmin_{\blfouriervec\in \rfield^\blnum}||
\observationvec - \samplingmat\bllaplacianevecmat\blfouriervec||^2
\\
&=
\bllaplacianevecmat
[\bllaplacianevecmat^T\samplingmat^T\samplingmat\bllaplacianevecmat]\inv
\bllaplacianevecmat^T\samplingmat^T\observationvec 
\label{eq:blestimateb}
\end{align}
\end{subequations}
where the second equality assumes that
$\bllaplacianevecmat^T\samplingmat^T\samplingmat\bllaplacianevecmat$
is invertible, a necessary and sufficient condition for the $\blnum$
entries of $\blfouriervec$ to be identifiable.

\cmt{Connection with kernel regression} The estimate $
\signallsestvec$ in \eqref{eq:blestimate} can be accommodated in the
kernel regression framework by properly constructing a
\emph{bandlimited kernel}.  \cmt{Laplacian kernel} Intuitively, one
can adopt a Laplacian kernel for which $r(\lambda_\vertexind)$ is
large if $\vertexind\notin \blset$ (cf.
Sec.~\ref{sec:laplaciankernels}). Consider the Laplacian kernel
$\blfullkernelmat$ with
\begin{align}
\label{eq:defrbl}
r_\beta(\lambda_\vertexind) =  
\begin{cases}
1/\beta & \vertexind \in \blset\\
\beta & \vertexind \notin \blset.
\end{cases}
\end{align}
For large $\beta$, this function  strongly penalizes
$\{\fourierrkhsfun_\vertexind\}_{\vertexind\notin \blset}$
(cf.~\eqref{eq:fourierregularizer}), which promotes bandlimited
estimates. The reason for setting $r(\lambda_\vertexind) = 1/\beta$
for $\vertexind \in \blset$ instead of $r(\lambda_\vertexind) =0$ is
to ensure that $\blfullkernelmat$ is non-singular, a property that
simplifies the statement and the proofs of some of the results in this
paper.

\begin{myproposition}
  \thlabel{th:blandkernel} Let $\signalridgeestvec$ denote the kernel
  ridge regression estimate from \eqref{eq:ridgeregressionestimate}
  with kernel $\blfullkernelmat$ as in \eqref{eq:defrbl} and
  $\mu>0$. If
  $\bllaplacianevecmat^T\samplingmat^T\samplingmat\bllaplacianevecmat$
  is invertible, as required by the estimator in
  \eqref{eq:blestimateb} for bandlimited signals, then
  $\signalridgeestvec\rightarrow \signallsestvec$ as $\beta
  \rightarrow \infty$.
\end{myproposition}
\begin{IEEEproof}
See Appendix~\ref{sec:proof:th:blandkernel}. 
\end{IEEEproof}

\cmt{Interpretation} \thref{th:blandkernel} shows that the framework
of kernel-based regression subsumes LS estimation of bandlimited
signals. A non-asymptotic counterpart of \thref{th:blandkernel} can be
found by setting $r_\beta(\lambda_\vertexind) = 0$ for $\vertexind \in
\blset$ in \eqref{eq:defrbl}, and noting that $\signalridgeestvec =
\signallsestvec$ if $\mu=0$. Note however that imposing $\mu=0$
renders $\signalridgeestvec$ a degenerate kernel-based estimate.

\subsubsection{Covariance kernels} 
\label{sec:covkernels}
\cmt{Motivation} So far, signal $\signalfun$ has been assumed
deterministic, which precludes accommodating certain forms of prior
information that probabilistic models can capture, such as domain
knowledge and historical data.  \cmt{section overview}A probabilistic
interpretation of kernel methods on graphs will be pursued here to
show that: (i) the optimal $\fullkernelmat$ in the MSE sense for ridge
regression is the covariance matrix of $\signalvec$; and, (ii)
kernel-based ridge regression seeks an estimate satisfying a system of
local LMMSE estimation conditions on a Markov random
field~\cite[Ch. 8]{bishop2006}.

\cmt{Connection KRR and LMMSE}

\cmt{Defs} Suppose without loss of generality that
$\{\signalfun(v_\vertexind)\}_{\vertexind=1}^{\vertexnum}$ are
zero-mean random variables. \cmt{LMMSE estimator}The LMMSE estimator
of $\signalvec$ given $\observationvec$ is the linear estimator
$\signallmmseestvec$ minimizing $\expectednb||\signalvec-
\signallmmseestvec ||_2^2$, where the expectation is over all
$\signalvec$ and noise realizations. With $\covmat
\define\expected{\signalvec \signalvec^T}$, the LMMSE estimate is
given by
\begin{align}
\label{eq:lmmseestimator}
\signallmmseestvec = \covmat \samplingmat^T[\samplingmat \covmat
\samplingmat^T + \noisevar \bm I_S]\inv \observationvec
\end{align}
where $\noisevar\define ({1}/{S})\expected{||\noisevec||_2^2}$ denotes
the noise variance.  \cmt{connection} Comparing
\eqref{eq:lmmseestimator} with \eqref{eq:ridgeregressionestimate} and
recalling that $\samplekernelmat\define \samplingmat \fullkernelmat
\samplingmat^T$, it follows that $\signallmmseestvec =
\signalridgeestvec$ with $\mu S = \noisevar$ and
$\fullkernelmat = \covmat$. In other words, the similarity measure
$\kernelmap(v_\vertexind,v_\vertexindp)$ embodied in the kernel map is
just the covariance $\mathop{\rm
  cov}[\signalfun(v_\vertexind),\signalfun(v_\vertexindp)]$.
\cmt{relative to~\cite{bazerque2013basispursuit}}A related
observation was pointed out in~\cite{bazerque2013basispursuit} for
general kernel methods.

\cmt{Generalizes \cite{gadde2015probabilistic}} In short, one can
interpret kernel ridge regression as the LMMSE estimator of a signal
$\signalvec$ with covariance matrix equal to $\fullkernelmat$. This
statement generalizes \cite[Lemma 1]{gadde2015probabilistic}, which
requires $\signalvec$ to be Gaussian, $\covmat$ rank-deficient, and
$\noisevar = 0$.

\cmt{Covariance kernel optimum in the MMSE sense} Recognizing that
kernel ridge regression is a linear estimator, readily establishes the
following result.
\begin{myproposition}
  \thlabel{prop:covkernel} If
  $\text{MSE}(\fullkernelmat,\regpar)\define \expectednb[
  ||\signalvec-\signalridgeestvec(\fullkernelmat,\regpar) ||^2 ]$,
  where $\signalridgeestvec(\fullkernelmat,\regpar)$ denotes the
  estimator in \eqref{eq:ridgeregressionestimate}, with kernel
  matrix $\fullkernelmat$, and regularization parameter $\regpar$, it
  then holds that
\begin{align*}
\text{MSE}(\covmat,\noisevar/\samplenum)\leq
\text{MSE}(\fullkernelmat,\regpar)
\end{align*}
for all kernel matrices $\fullkernelmat$ and $\regpar>0$. 
\end{myproposition}

\cmt{Implications} Thus, for criteria aiming to minimize the MSE,
\thref{prop:covkernel} suggests $\fullkernelmat$ to be chosen
\emph{close} to $\covmat$. This observation may be employed for kernel
selection and for parameter tuning in graph signal reconstruction
methods of the kernel ridge regression family (e.g. the Tikhonov
regularized estimators
from~\cite{narang2013localized,belkin2004regularization,shuman2013emerging}
and~\cite[eq. (27)]{chen2015recovery}), whenever an estimate of
$\covmat$ can be obtained from historical data. For instance, the
function $r$ involved in Laplacian kernels can be chosen such that
$\fullkernelmat$ resembles $\covmat$ in some sense. Investigating such
approaches goes beyond the scope of this paper.

\cmt{Interpretation on MRF} A second implication of the connection
between kernel ridge regression and LMMSE estimation involves signal
estimation on Markov random fields~\cite[Ch. 8]{bishop2006}.
\cmt{MRF} In this class of graphical models, an edge connects
$v_\vertexind$ with $v_\vertexindp$ if $\signalfun(v_\vertexind)$ and
$\signalfun(v_\vertexindp)$ are \emph{not} independent given
$\{\signalfun(v_\vertexindpp)\}_{ \vertexindpp \neq
  \vertexind,\vertexindp}$.  Thus, if $v_\vertexindp\notin
\mathcal{N}_\vertexind$, then $\signalfun(v_\vertexind)$ and
$\signalfun(v_\vertexindp)$ are independent given
$\{\signalfun(v_\vertexindpp)\}_{ \vertexindpp \neq
  \vertexind,\vertexindp}$. In other words, when
$\signalfun(v_\vertexindpp)$ is known for all neighbors $v_\vertexindpp\in
\mathcal{N}_\vertexind$, function values at non-neighboring vertices
do not provide further information. This spatial Markovian property
motivates the name of this class of graphical models. Real-world
graphs obey this property when the topology captures direct
interaction, in the sense that the interaction between the entities
represented by two non-neighboring vertices $v_\vertexind$ and
$v_\vertexindp$ is necessarily through vertices in a \emph{path}
connecting $v_\vertexind$ with~$v_\vertexindp$.

\cmt{proposition}
  \begin{myproposition}
    \thlabel{prop:markov} Let $\mathcal{G}$ be a Markov random field,
    and consider the estimator in \eqref{eq:ridgeregressionestimate}
    with $\fullkernelmat = \covmat \define\expected{\signalvec
      \signalvec^T}$, and $\regpar = \noisevar/\samplenum$. Then, it
    holds that
    \begin{align}
      \label{eq:ridgemarkov}
      \signalridgeestfun(v_\vertexind) =
      \begin{cases}
        \text{LMMSEE}\left[\signalfun(v_\vertexind)\Big|
          \{\signalridgeestfun(v)\}_{v \in
            \mathcal{N}_\vertexind}\right]&\text{if } \vertexind \notin \samplingset\\
        y_{\sampleind(\vertexind)} -  \hat \noisesamp_{\sampleind(\vertexind)}&\text{if } \vertexind \in \samplingset
      \end{cases}
    \end{align}
    for $\vertexind=1,\ldots,\vertexnum$, where
    $\sampleind(\vertexind)$ denotes the sample index of the observed
    vertex $v_\vertexind$, i.e., $y_{\sampleind(\vertexind)} =
    \signalfun(v_\vertexind)+ \noisesamp_{\sampleind(\vertexind)}$,
    and
    \begin{align*}
      \hat \noisesamp_{\sampleind(\vertexind)} =
      \frac{\noisevar}{\sigma^2_{\vertexind|\mathcal{N}_\vertexind}}
      \Big[\signalridgeestfun(v_\vertexind) -
      \text{LMMSEE}\left[\signalfun(v_\vertexind)\Big|
        \{\signalridgeestfun(v)\}_{v \in
          \mathcal{N}_\vertexind} \right]\Big].
    \end{align*}
    Here, $\text{LMMSEE}[\signalfun(v_\vertexind)|
    \{\signalridgeestfun(v)\}_{v \in
      \mathcal{N}_\vertexind}]$ is the LMMSE estimator of
    $\signalfun(v_\vertexind)$ given
    $\signalfun(v_\vertexindp)=\signalridgeestfun(v_\vertexindp)$,
    $v_\vertexindp \in \mathcal{N}_\vertexind$, and
    $\sigma^2_{\vertexind|\mathcal{N}_\vertexind}$ is the variance of this
    estimator.
  \end{myproposition}
  \begin{proof}
    See Appendix~\ref{app:proof:prop:markov}. 
  \end{proof}

 \cmt{interpretation}

 If a (noisy) observation of $\signalfun$ at $v_\vertexind$ is not
 available, i.e. $\vertexind \notin \samplingset$, then kernel ridge
 regression finds $\signalridgeestfun(v_\vertexind)$ as the LMMSE
 estimate of $\signalfun(v_\vertexind)$ given function values at the
 neighbors of $v_\vertexind$. However, since the latter are not
 directly observable, their ridge regression estimates are used
 instead.  Conversely, when $v_\vertexind$ is observed, implying that
 a sample $y_{\sampleind(\vertexind)}$ is available, the sought
 estimator subtracts from this value an estimate $\hat
 \noisesamp_{\sampleind(\vertexind)}$ of the observation noise
 $\noisesamp_{\sampleind(\vertexind)}$.  Therefore, the kernel
 estimate on a Markov random field seeks an estimate satisfying the
 system of \emph{local LMMSE conditions} given by
 \eqref{eq:ridgemarkov} for $\vertexind =1,\ldots,\vertexnum$.

 \numberedremark{rem:relaxmrf}\textbf{.} \cmt{relax $\mathcal{G}$ is
   MRF} In \thref{prop:markov}, the requirement that $\mathcal{G}$ is
 a Markov random field can be relaxed to that of being a
 \emph{conditional correlation graph}, defined as a graph where
 $(v_\vertexind,v_\vertexindp)\in \edgeset$ if
 $\signalfun(v_\vertexind)$ and $\signalfun(v_\vertexindp)$ are
 correlated given $\{\signalfun(v_\vertexindpp)\}_{\vertexindpp\neq
   \vertexind,\vertexindp}$. Since correlation implies dependence, any
 Markov random field is also a conditional correlation graph. A
 conditional correlation graph can be constructed from $\covmat\define
 \expectednb[\signalvec\signalvec^T]$ by setting $ \edgeset
 =\{(v_\vertexind,v_\vertexindp):(\covmat\inv)_{\vertexind,\vertexindp}\neq
 0\} $ (see e.g.~\cite[Th. 10.2]{kay1}).

 \numberedremark{rem:nonmrf}\textbf{.} \cmt{$\mathcal{G}$ not MRF}
 Suppose that kernel ridge regression is adopted to estimate a
 function $\signalfun$ on a certain graph $\mathcal{G}$, not
 necessarily a Markov random field, using a kernel $\fullkernelmat
 \neq \covmat \define \expectednb[\signalvec\signalvec^T]$. Then it
 can still be interpreted as a method applying \eqref{eq:ridgemarkov}
 on a conditional correlation graph $\mathcal{G}'$ and adopting a
 signal covariance matrix $\fullkernelmat$.

\subsubsection{Further kernels}
\label{sec:otherkernels}
Additional signal reconstructors can be interpreted as kernel-based
regression methods for certain choices of $\fullkernelmat$.
Specifically, it can be seen that~\cite[eq. (27)]{chen2015recovery} is
tantamount to kernel ridge regression with kernel
\begin{align*}
\fullkernelmat =[(\bm I_\vertexnum -\adjacencymat)^T(\bm I_\vertexnum -\adjacencymat)]\inv
\end{align*}
provided that the adjacency matrix $\adjacencymat$ is properly scaled so that
this inverse exists.  Another example is the Tikhonov regularized
estimate in \cite[eq. (15)]{narang2013localized}, which is recovered
as kernel ridge regression upon setting
\begin{align*}
\fullkernelmat =[\bm H^T\bm H + \epsilon \bm I_\vertexnum]\inv
\end{align*}                    %
and letting $\epsilon>0$ tend to 0, where $\bm H$ can be viewed as a
\emph{high-pass filter} matrix. The role of the term $\epsilon \bm
I_\vertexnum$ is to ensure that the matrix within brackets is
invertible.

\subsection{Kernel-based smoothing and graph filtering} 
\label{sec:smoothing}

\cmt{Overview} When an observation $y_\vertexind$ is available per
vertex $v_\vertexind$ for $\vertexind =1,\ldots,\vertexnum$, kernel
methods can still be employed for denoising purposes. Due to the
regularizer in~\eqref{eq:single-general}, the estimate $\signalestvec$
will be a smoothed version of $\bm y$. This section shows how ridge
regression smoothers can be thought of as graph filters, and vice
versa. The importance of this two-way link is in establishing that
kernel smoothers can be implemented in a decentralized fashion as
graph filters~\cite{shuman2013emerging}.

   Upon setting $\samplingmat = \bm I_\vertexnum$ in
  \eqref{eq:ridgeregressionestimate}, one recovers the ridge
  regression smoother $\signalridgesmoothvec = \fullkernelmat
  (\fullkernelmat + \regpar \vertexnum \bm{I}_\vertexnum )^{-1}
  \bm{y}$. If $\fullkernelmat$ is a Laplacian kernel, then
\begin{align}
\label{eq:smoothinglaplacian}
\signalridgesmoothvec = \laplacianevecmat \smoothfun(\bm \Lambda)   \laplacianevecmat^T \bm{y}
\end{align}
where $ \smoothfun(\lambda)\define { \indicator[r(\lambda)\neq 0]/
}{[ 1+\regpar \vertexnum r(\lambda) ]} $.

 To see how \eqref{eq:smoothinglaplacian} relates to a graph
filter, recall that the latter is an operator assigning $\bm y\mapsto
\bm y_F$, where~\cite{shuman2013emerging}
\begin{subequations}
\begin{align}
\label{eq:graphfiltervertexdomain}
  \bm y_F &\define \left(\filtercoef_0\bm I_\vertexnum+
    \sum_{\vertexind=1}^{\vertexnum-1} \filtercoef_\vertexind \bm
    L^\vertexind\right)\bm y 
\\&= \laplacianevecmat
 \left(\filtercoef_0\bm I_\vertexnum+
    \sum_{\vertexind=1}^{\vertexnum-1}\filtercoef_\vertexind \bm \Lambda^\vertexind\right)
\laplacianevecmat^T \bm{y}.
\label{eq:graphfilterfourierdomain}
\end{align}
\end{subequations}
\cmt{distributed implementation} Graph filters can be implemented in a
decentralized fashion since \eqref{eq:graphfiltervertexdomain}
involves successive products of $\bm y$ by $\bm L$ and these products
can be computed at each vertex by just exchanging information with
neighboring vertices.  \cmt{Freq. resp}Expression
\eqref{eq:graphfilterfourierdomain} can be rewritten in the
\emph{Fourier domain} (cf. Sec.~\ref{sec:laplaciankernels}) as $ \tbm
y_F=[\filtercoef_0\bm I_\vertexnum+
\sum_{\vertexind=1}^{\vertexnum-1}\filtercoef_\vertexind \bm
\Lambda^\vertexind] \tbm{y}$ upon defining $ \tbm y_F
\define\laplacianevecmat^T\bm y_F$ and $ \tbm y
\define\laplacianevecmat^T\bm y$. For this reason, the diagonal of
$\filtercoef_0\bm I_\vertexnum+
\sum_{\vertexind=1}^{\vertexnum-1}\filtercoef_\vertexind \bm
\Lambda^\vertexind$ is referred to as the \emph{frequency response} of
the filter.

\cmt{smoother is filter} Comparing \eqref{eq:smoothinglaplacian} with
\eqref{eq:graphfilterfourierdomain} shows that $\signalridgesmoothvec$
can be interpreted as a graph filter with frequency response
${\smoothfun(\bm \Lambda)}$. Thus, implementing
$\signalridgesmoothvec$ in a decentralized fashion using
\eqref{eq:graphfiltervertexdomain} boils down to solving for
$\{\filtercoef_\vertexind\}_{\vertexind = 1}^\vertexnum$ the system of
linear equations $\{\filtercoef_0+
\sum_{\vertexindp=1}^{\vertexnum-1}\filtercoef_\vertexindp
\lambda_\vertexind^\vertexindp =
\smoothfun(\lambda_\vertexind)\}_{\vertexind=1}^\vertexnum$.  \cmt{
  filter is smoother} Conversely, given a filter, a Laplacian kernel
can be found so that filter and smoother coincide.  To this end,
assume without  loss of generality that
$\fourierfiltdiag_\vertexind\leq 1~\forall \vertexind$, where
$\fourierfiltdiag_\vertexind\define \filtercoef_0+
\sum_{\vertexindp=1}^{\vertexnum-1}\filtercoef_\vertexindp
\lambda_\vertexind^\vertexindp $; otherwise, simply scale
$\{\filtercoef_\vertexind\}_{\vertexind=0}^{\vertexnum-1}$. Then,
given $\{\filtercoef_\vertexind\}_{\vertexind=0}^{\vertexnum-1}$, the
sought kernel can be constructed by setting
\begin{align*}
 r(\lambda_\vertexind)=\frac{
1-\fourierfiltdiag_\vertexind
}
{\regpar \vertexnum \fourierfiltdiag_\vertexind}
~\indicator[\fourierfiltdiag_\vertexind\neq
   0].
\end{align*}


\section{Multi-kernel  graph signal reconstruction}
\label{sec:mkl}

\begin{figure}[t]
 \centering \includegraphics[width=.45\textwidth]{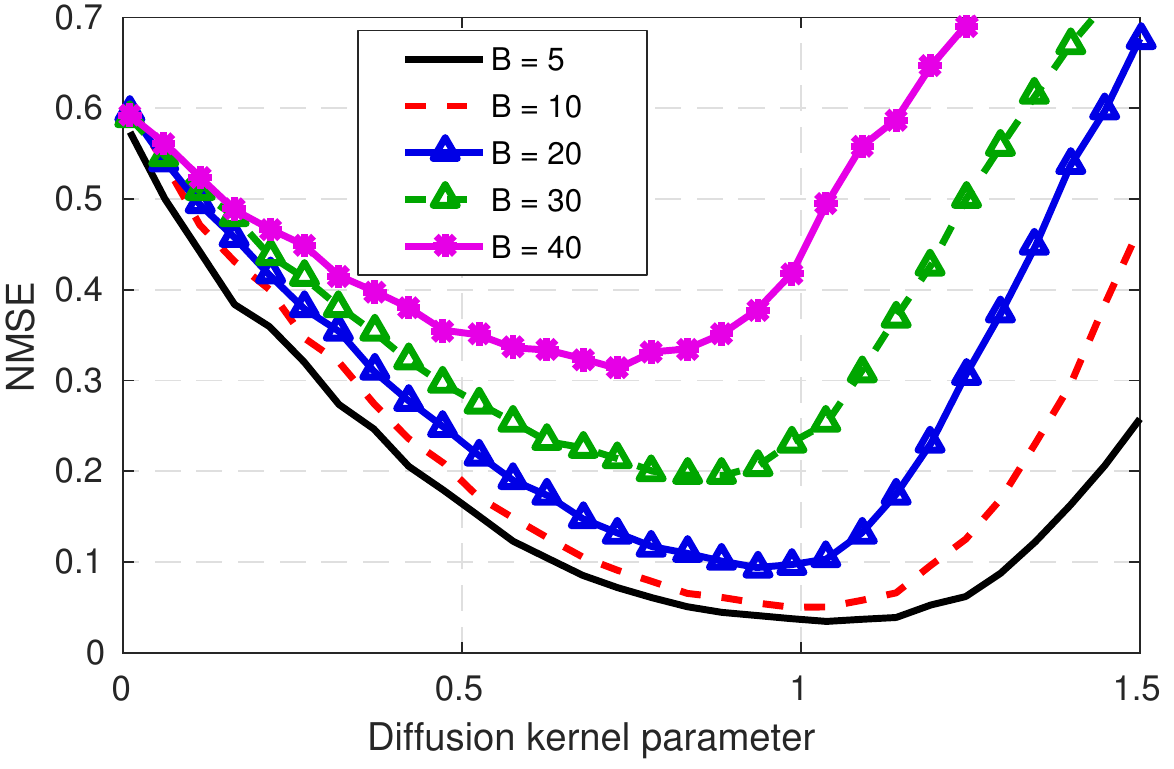}
 \caption{Influence of the diffusion kernel parameter $\sigma^2$ on
   NMSE for $\samplenum = 40$ and several bandwidths $B$
   ($\vertexnum=100$, $\text{SNR} = 20 \text{ dB}$, $\mu = 10^{-4}$).
 }
 \label{fig:nmse_vs_sigma}
\end{figure}

\cmt{Motivation}One of the limitations of kernel methods is their
sensitivity to the choice of the kernel. To appreciate this,
\cmt{kernel choice} Fig.~\ref{fig:nmse_vs_sigma} depicts the
normalized mean-square error (NMSE)
$\expectednb||\signalvec-\signalestvec ||_2^2/\expectednb||\signalvec||_2^2$
when $\mathcal{L}$ is the square loss and $\Omega(\zeta) = |\zeta|$
across the parameter $\sigma^2$ of the adopted diffusion kernel (see
Sec.~\ref{sec:laplaciankernels}). The simulation setting is described
in Sec.~\ref{sec:numericaltests}. At this point though, it suffices to
stress the impact of $\sigma^2$ on the NMSE and the dependence of
the optimum $\sigma^2$  on the bandwidth $B$ of
$\signalfun$. 

\cmt{bandwidth choice} Similarly, the performance of estimators for
bandlimited signals degrades considerably if the estimator assumes a
frequency support $\blset$ that differs from the actual one.  Even for
estimating \emph{low-pass signals}, for which $\blset = \{1,\ldots,
\blnum\}$, parameter $\blnum$ is unknown in practice. Approaches for
setting $\blnum$ were considered
in~\cite{narang2013structured,anis2016proxies}, but they rely solely
on $\samplingset$ and $\bm L$, disregarding the observations $\bm
y$. Note that by adopting the bandlimited kernels
from Sec.~\ref{sec:bandlimited}, bandwidth selection boils down to kernel
selection, so both problems will be treated jointly in the
sequel through the lens of kernel-based learning.

\cmt{kernel methds\ra MKL} This section advocates an MKL approach to
kernel selection in graph signal reconstruction. \cmt{two criteria}Two
algorithms with complementary strengths will be developed. Both select
the most suitable kernels within a user-specified \emph{kernel
  dictionary}.

\subsection{RKHS superposition}
\label{sec:rkhssuperposition}

   \cmt{form of the estimate} Since $\mathcal{H}$ in
  \eqref{eq:single-general} is determined by $\kernelmap$, kernel selection is
  tantamount to RKHS selection. Therefore, a kernel dictionary
  $\{\kernelmap_m\}_{\rkhsind=1}^\rkhsnum$ can be equivalently thought of
  as an RKHS dictionary $\{\mathcal{H}_m\}_{\rkhsind=1}^\rkhsnum$,
  which motivates estimates of the form
\begin{equation}
\label{eq:sumh}
{\rkhsestfun} = \sum_{m=1}^\rkhsnum {\rkhsestfun}_m, \quad {\rkhsestfun}_m \in \mathcal{H}_m.
\end{equation}
\cmt{criterion} Upon adopting a criterion that controls sparsity in
this expansion, the ``best'' RKHSs will be selected. A reasonable
approach is therefore to generalize \eqref{eq:single-general} to
accommodate multiple RKHSs. With $\mathcal{L}$ selected as the square
loss and $\Omega(\zeta)=|\zeta|$, one can pursue an estimate
$\rkhsestfun$ by solving
\begin{equation}
  \hspace{-.25cm}\min_{\{\rkhsfun_\rkhsind\in \mathcal{H}_m\}_{m=1}^M} \frac{1}{S} \sum_{s=1}^S \left[y_s - \sum_{m=1}^M
    \rkhsfun_\rkhsind(v_{\vertexind_s}) \right]^2 + \regpar \sum_{m=1}^M\|\rkhsfun_\rkhsind\|_{\mathcal{H}_m}.
	\label{eq:multi-sq}
\end{equation}

 \cmt{Representer theorem} Invoking \thref{prop:representer}
per $\rkhsfun_\rkhsind$ establishes that the minimizers of
\eqref{eq:multi-sq} can be written as
\begin{align}
\label{eq:representer-mkl-g}
\hat \rkhsfun_\rkhsind(v) = \sum_{s=1}^{S} \alpha^m_s \kernelmap_m(v,v_{\vertexind_s}) ,~~~~m=1,\ldots,M
\end{align}
for some coefficients $\alpha^m_s$. Substituting
\eqref{eq:representer-mkl-g} into \eqref{eq:multi-sq} suggests
obtaining these coefficients as
\begin{equation}
\argmin_{\{\samplealphavec_\rkhsind\}_{m=1}^M} \frac{1}{S} \left\| \bm{y}
- \sum_{m=1}^M\samplekernelmat_m\samplealphavec_\rkhsind \right\|^2 + \regpar \sum_{m=1}^M \left(
\samplealphavec_\rkhsind^T \samplekernelmat_m \samplealphavec_\rkhsind \right)^{1/2}
	\label{eq:multi-vector-sq}
\end{equation}
where $\samplealphavec_\rkhsind:=[\alpha_1^m,\ldots,\alpha_S^m]^T$,
and $\samplekernelmat_\rkhsind=\samplingmat \fullkernelmat_m
\samplingmat^T$ with
$(\fullkernelmat_m)_{\vertexind,\vertexindp}:=\kernelmap_m(v_{\vertexind},v_{\vertexindp})$.
\cmt{promotes sparsity} Letting $\trsamplealphavec_m :=
\samplekernelmat_m^{1/2}\samplealphavec_m$, expression
\eqref{eq:multi-vector-sq} becomes
\begin{align}
&\argmin_{\{\trsamplealphavec_m\}_{m=1}^M} \frac{1}{S} \left\| \bm{y} -
\sum_{m=1}^M\samplekernelmat_m^{1/2}\trsamplealphavec_m \right\|^2 + \regpar \sum_{m=1}^M \| \trsamplealphavec_m
\|_2 	\label{eq:multi-1norm}.
\end{align}
Note that the sum in the regularizer of \eqref{eq:multi-1norm} can be
interpreted as the $\ell_1$-norm of $ [||\trsamplealphavec_1||_2,
\ldots, ||\trsamplealphavec_M||_2]^T$, which is known to promote
sparsity in its entries and therefore in \eqref{eq:sumh}. Indeed,
\eqref{eq:multi-1norm} can be seen as a particular instance of group
Lasso~\cite{bazerque2013basispursuit}.

\begin{algorithm}[t]                
 	\caption{ADMM for multi-kernel regression}
 	\label{algo:admmc1}    
 	\begin{minipage}{20cm}
 		\begin{algorithmic}[1]
	 		\STATE Input: $ \rho, \epsilon>0,
                        \trsamplealphavec\iternot{0}, \bm \nu^0 $ 
 			\REPEAT
 			\STATE  $
                        \trsamplealphavec_\rkhsind\iternot{k+1} =
                        \softthresholding_{\mu
                          S/2\rho}(\auxveco_\rkhsind\iternot{k} +
                        \bm \nu_\rkhsind\iternot{k}) $~~~~~$\rkhsind
                        =1,\ldots,\rkhsnum$ 
 			\STATE  $ \auxveco\iternot{k+1} =(\auxmatf^T\auxmatf +
                        \rho \bm{I})^{-1}[\auxmatf^T\observationvec +
                        \rho(\trsamplealphavec\iternot{k+1} - \bm \nu\iternot{k})] $ 
 			\STATE  $\bm  \nu_\rkhsind\iternot{k+1} = \bm
                        \nu_\rkhsind\iternot{k} + 
                        \auxveco_\rkhsind\iternot{k+1} -
                        \trsamplealphavec_\rkhsind\iternot{k+1},
                        $~~~$\rkhsind =1,\ldots,\rkhsnum$
                        \STATE $k \leftarrow k+1$
 			\UNTIL{ $|| \auxveco\iternot{k+1} -
                          \trsamplealphavec\iternot{k+1}|| \leq \epsilon $ }
 		\end{algorithmic}
 	\end{minipage}
\end{algorithm}

\cmt{ADMM} As shown next, \eqref{eq:multi-1norm} can be efficiently
solved using the alternating-direction method of multipliers
(ADMM)~\cite{bazerque2011splines}. To this end, rewrite
\eqref{eq:multi-1norm} by defining $ \auxmatf \define
[\samplekernelmat_1^{1/2}, \ldots, \samplekernelmat_M^{1/2}] $ and $
\trsamplealphavec \define [\trsamplealphavec_1^T, \ldots,
\trsamplealphavec_M^T]^T $, and introducing the auxiliary variable
$\auxveco\define [\auxveco_1^T,\ldots,\auxveco_\rkhsnum^T]^T$, as
\begin{equation}
\label{eq:admmform}
  \begin{array}{rl}
    \displaystyle\min_{\trsamplealphavec, \auxveco}
    &\;\displaystyle \frac{1}{2} \|\observationvec - \auxmatf\auxveco\|^2 + 
\frac{S\mu}{2}\sum_{\rkhsind=1}^\rkhsnum
 \| \trsamplealphavec_\rkhsind\|_2 \\
    \textrm{s. to} &\; \trsamplealphavec - \auxveco = \bm 0.
  \end{array}
\end{equation}
ADMM iteratively minimizes the \emph{augmented Lagrangian} of
\eqref{eq:admmform} with respect to $\trsamplealphavec$ and $\auxveco$
in a block-coordinate descent fashion, and updates the Lagrange
multipliers associated with the equality constraint using gradient
ascent (see~\cite{giannakis2016decentralized} and references
therein). The resulting iteration is summarized as
Algorithm~\ref{algo:admmc1}, where $ \rho $ is the augmented
Lagrangian parameter, $ \bm \nu\define[\bm \nu_1^T,\ldots,\bm
\nu_\rkhsnum^T]^T $ is the Lagrange multiplier associated with the
equality constraint, and
\begin{equation*}
 \softthresholding_\zeta(\bm a) \define
\frac{\max(0,||\bm a||_2-\zeta)}{||\bm a||_2} \bm a
\end{equation*}
is the so-called \emph{soft-thresholding} operator~\cite{bazerque2011splines}.

\cmt{final estimate} After obtaining
${\{\trsamplealphavec_m\}_{m=1}^M}$ from Algorithm~\ref{algo:admmc1},
the wanted function estimate can be recovered as
\begin{align}
\label{eq:mklrkhsreconstruction}
  \signalestvec = \sum_{m=1}^M\fullkernelmat_m \samplingmat^T
  \samplealphavec_m= \sum_{m=1}^M\fullkernelmat_m \samplingmat^T
  \samplekernelmat_m^{-1/2}\trsamplealphavec_m.
\end{align}

\cmt{balancing} It is recommended to normalize the kernel matrices
in order to prevent imbalances in the kernel selection. Specifically,
one can scale $\{\fullkernelmat_\rkhsind\}_{\rkhsind=1}^\rkhsnum$ such
that $\tr{\fullkernelmat_\rkhsind}=1~\forall\rkhsind$. If
$\fullkernelmat_\rkhsind$ is a Laplacian kernel (see
Sec.~\ref{sec:laplaciankernels}), where $\fullkernelmat_m =
\laplacianevecmat r_m^\dagger(\bm \Lambda) \laplacianevecmat^T$, one
can scale $r_m$ to ensure $\sum_{\vertexind=1}^\vertexnum
r_\rkhsind^\dagger(\lambda_\vertexind) = 1$.

\numberedremark{rem:mklliterature}\textbf{.}  Although criterion
\eqref{eq:multi-sq} is reminiscent of the MKL approach
of~\cite{bazerque2013basispursuit}, the latter differs markedly
because it assumes that the right-hand side of \eqref{eq:sumh} is
uniquely determined given $\signalestfun$, which allows application of
\eqref{eq:single-general} over a direct-sum RKHS
$\mathcal{H}:=\mathcal{H}_1\oplus\cdots\oplus\mathcal{H}_N$ with an
appropriately defined norm. However, this approach cannot be pursued
here since RKHSs of graph signals frequently overlap, implying that
their sum is not a direct one (cf. discussion after~\eqref{eq:norm}).

\subsection{Kernel superposition}

  \begin{algorithm}[!t]
  \caption{Interpolated Iterative Algorithm}
  \label{algo:iia}    
  \begin{minipage}{20cm}
  \begin{algorithmic}[1]
  \STATE Input: $\kernelcoefvec\iternot{0}$, 
  $\{\samplekernelmat_\rkhsind\}_{\rkhsind=1}^\rkhsnum$, $\regpar$ , $\kernelcoefvec_0$, $\radius$,
 $\eta$, $\epsilon$.
  \STATE  $ \samplealphavec\iternot{0} =
  (\samplekernelmat(\kernelcoefvec\iternot{0})+\regpar\samplenum \bm{I})^{-1} \observationvec $
\STATE $k=0$
  \REPEAT
	  \STATE  $ \auxvect\iternot{k} =
          [\samplealphavec\iternotT{k}\samplekernelmat_0\samplealphavec\iternot{k},
          \ldots,
          \samplealphavec\iternotT{k}\samplekernelmat_\rkhsnum\samplealphavec\iternot{k}]^T
          $ 
	  \STATE  $ \kernelcoefvec\iternot{k} = \kernelcoefvec_0 + (\radius
/{\|\auxvect\iternot{k}\|_2})           ~{\auxvect\iternot{k}} $ 
	  \STATE  $ \samplealphavec\iternot{k+1} =
          \eta\samplealphavec\iternot{k} + (1-\eta)
          [\samplekernelmat(\kernelcoefvec\iternot{k}) + \regpar\samplenum \bm{I}]^{-1} \observationvec $ 
          \STATE $k\leftarrow k+1$
  \UNTIL{ $ \| \samplealphavec\iternot{k+1} - \samplealphavec\iternot{k} \| < \epsilon $ }
  \end{algorithmic}
  \end{minipage}
  \end{algorithm}

  \cmt{Overview} The MKL algorithm in Sec.~\ref{sec:rkhssuperposition}
  can identify the best subset of RKHSs and therefore kernels, but
  entails $\rkhsnum\samplenum $ unknowns
  (cf. \eqref{eq:multi-vector-sq}). This section introduces an
  alternative approach entailing only $\rkhsnum+\samplenum$
  variables at the price of not guaranteeing a sparse kernel
  expansion.

  \cmt{idea} The approach is to postulate a kernel of the form
  $ \fullkernelmat(\kernelcoefvec) = \sum_{\rkhsind=1}^{\rkhsnum}
  \kernelcoef_\rkhsind \fullkernelmat_\rkhsind $, where
  $\{\fullkernelmat_\rkhsind\}_{\rkhsind=1}^\rkhsnum$ is given and $
  \kernelcoef_\rkhsind\geq 0~\forall \rkhsind$. The coefficients
  $\kernelcoefvec\define[\kernelcoef_1,\ldots,\kernelcoef_\rkhsnum]^T$
  can be found by jointly minimizing
  \eqref{eq:solutiongraphsmatrixform} with respect to $\kernelcoefvec$
  and $\samplealphavec$~\cite{micchelli2005function}
\begin{equation}
  (\kernelcoefvec, \samplealphaestvec) := \argmin_{\kernelcoefvec, \samplealphavec} \frac{1}{S} 
  \mathcal{L} (\bm  y - \samplekernelmat(\kernelcoefvec) \samplealphavec)
  + \regpar \Omega((\samplealphavec^T \samplekernelmat(\kernelcoefvec) \samplealphavec)^{1/2})
	\label{eq:solutiongraphsmatrixformmkl}
\end{equation}
where $\samplekernelmat(\kernelcoefvec) \define \samplingmat
\fullkernelmat(\kernelcoefvec) \samplingmat^T$. Except for degenerate
cases, problem \eqref{eq:solutiongraphsmatrixformmkl} is not jointly
convex in $\kernelcoefvec$ and $ \samplealphaestvec$, but it is
separately convex in these variables for a convex
$\mathcal{L}$~\cite{micchelli2005function}.
\cmt{contribution}Criterion \eqref{eq:solutiongraphsmatrixformmkl}
generalizes the one in~\cite{argyriou2005combining}, which aims at
combining Laplacian matrices of multiple graphs sharing the same
vertex set.

 \cmt{Cortes IIA} A method termed \emph{interpolated iterative
  algorithm} (IIA) was proposed in \cite{cortes2009regularization} to
solve \eqref{eq:solutiongraphsmatrixformmkl} when $\mathcal{L}$ is the
square loss, $\Omega(\zeta) = \zeta^2$, and $\kernelcoefvec$ is
constrained to lie in a ball $\bm{\Theta}:=\{ \kernelcoefvec:
\kernelcoefvec \geq \bm 0 \textrm{ and }
\|\kernelcoefvec-\kernelcoefvec_0\| \leq \radius \}$ for some
user-defined center $\kernelcoefvec_0$ and radius $\radius>0$. This
constraint ensures that $\kernelcoefvec$ does not diverge. The
first-order optimality conditions for
\eqref{eq:solutiongraphsmatrixformmkl} yield a nonlinear system of
equations, which IIA solves iteratively. This algorithm is displayed
as Algorithm~\ref{algo:iia}, where $\eta>0$ is the step size.

\cmt{Laplacian kernel smoothing} As a special case, it is worth noting
that Algorithm~\ref{algo:iia} enables kernel selection in ridge
smoothing, which is tantamount to optimal filter selection for graph
signal denoising (cf. Sec.~\ref{sec:smoothing}). In this case,
Algorithm~\ref{algo:iia} enjoys a particularly efficient
implementation for Laplacian kernels since their kernel matrices share
eigenvectors. Specifically, recalling that for smoothing
$\samplekernelmat_\rkhsind= \fullkernelmat_\rkhsind =
\laplacianevecmat r_\rkhsind ^\dagger(\bm \Lambda) \laplacianevecmat^T
$ and letting $\fourieralphavec \define
[\fourieralpha_1,\ldots,\fourieralpha_\vertexnum]^T\define
\laplacianevecmat^T\samplealphavec$, suggests that the
$\samplealphavec$-update in Algorithm~\ref{algo:iia} can be replaced
with its scalar version
\begin{align*}
\fourieralpha\iternot{k+1}_\vertexind = \eta\fourieralpha_\vertexind\iternot{k} +
 \frac{(1-\eta)\fourierobservation_\vertexind}{\sum_{\rkhsind=1}^\rkhsnum \kernelcoef_\rkhsind\iternot{k}
  r_\rkhsind^\dagger(\lambda_\vertexind) + \regpar\samplenum
 } ,~\vertexind =1,\ldots,\vertexnum
\end{align*}
whereas the $\auxvect$-update can be replaced with $
\auxentt_\rkhsind\iternot{k} = \sum_{\vertexind=1}^\vertexnum
r_\rkhsind^\dagger(\lambda_\vertexind)(\fourieralpha_\vertexind\iternot{k})^2
$, where $ \auxvect \define [\auxentt_1,\ldots,\auxentt_\rkhsnum]^T$.

 \section{Numerical tests}
 \label{sec:numericaltests}

 This section compares the proposed methods with competing
 alternatives in synthetic- as well as real-data experiments. Monte
 Carlo simulation is used to average performance metrics across
 realizations of the signal $\signalvec$, noise $\noisevec$ (only for
 synthetic-data experiments), and sampling set $\samplingset$. The
 latter is drawn uniformly at random without replacement from
 $\{1,\ldots,\vertexnum\}$.

\subsection{Synthetic bandlimited signals}
\label{sec:syntheticdata}

Three experiments were conducted on an
Erd\H{o}s-R$\grave{\text{e}}$nyi random graph with probability of edge
presence $0.25$~\cite{kolaczyck2009}. Bandlimited signals were
generated as in~\eqref{eq:defbandlimited} with $\blset=\{1,\ldots,B\}$
for a certain $B$. The coefficients
$\{\fourierrkhsfun_\vertexind\}_{\vertexind \in \blset}$ are
independent uniformly distributed over the interval $[0,1]$.  Gaussian
noise was added to yield a target \emph{signal-to-noise} ratio
$\text{SNR} \define ||\signalvec||^2/(\vertexnum \noisevar)$.

The first experiment was presented in Fig.~\ref{fig:nmse_vs_sigma} and
briefly described in Sec.~\ref{sec:mkl} to illustrate the strong
impact of the kernel choice on the $\text{NMSE}\define
\expectednb||\signalvec-\signalestvec
||_2^2/\expectednb||\signalvec||_2^2$.

\begin{figure}[t]
 \centering \includegraphics[width=.45\textwidth]{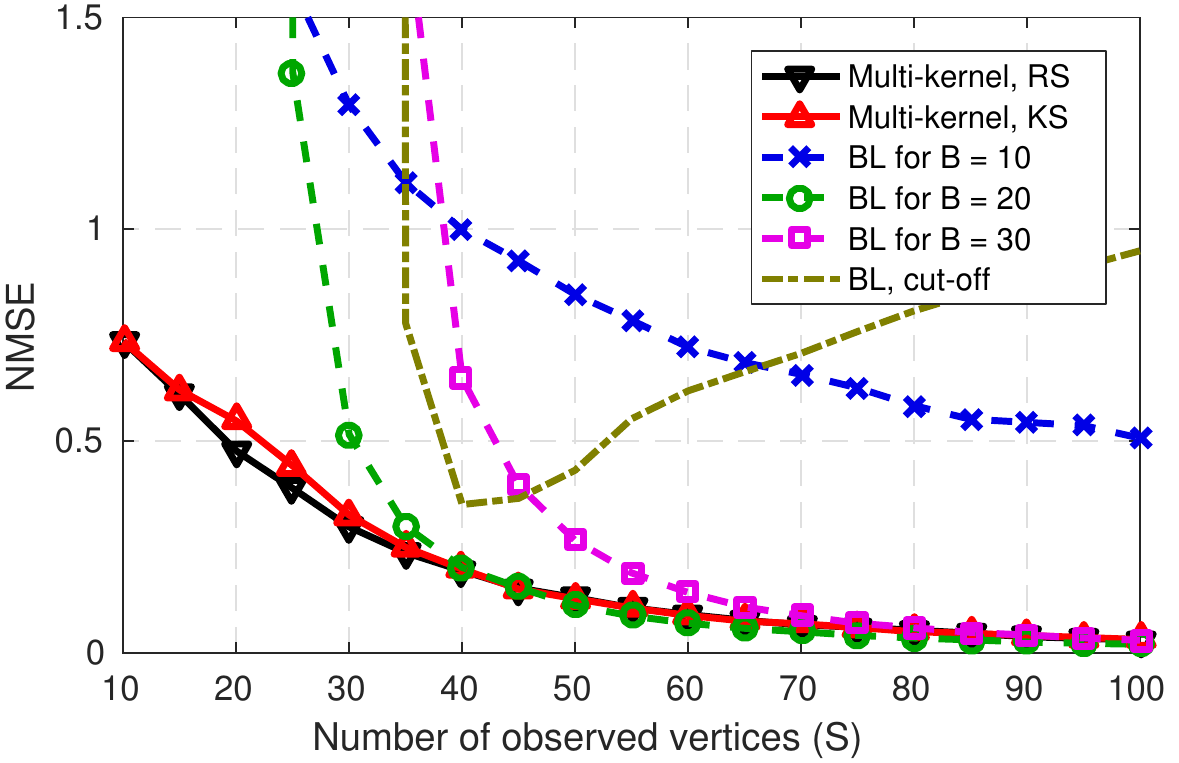}
 \caption{Comparison of different algorithms for estimating
   bandlimited signals. Per Monte Carlo iteration, a bandlimited
   signal with $B=20$ is generated ($\vertexnum=100$,
   $\text{SNR}=10~\text{dB}$).  }
 \label{fig:bandlimited}
\end{figure}

The second experiment compares methods for estimating bandlimited
signals. Fig.~\ref{fig:bandlimited} depicts the NMSE in reconstructing  a
bandlimited signal with $B = 20$ across $\samplenum$. 
The first two curves correspond to the MKL approaches proposed in
Sec.~\ref{sec:mkl}, which employ a dictionary with $5$ bandlimited
kernels, where the $\rkhsind$-th kernel has $\beta = 10^4$ and
bandwidth $5\rkhsind +5$, $\rkhsind = 1,\ldots,5$. The regularization
parameter $\mu$ was set to $10^{-1}$ for RKHS superposition (RS), and
to $5\cdot 10^{-3}$ for kernel superposition (KS).  The next three
curves correspond to the LS estimator for bandlimited (BL) signals
in~\eqref{eq:blestimateb}~\cite{narang2013localized,gadde2015probabilistic,tsitsvero2016uncertainty,chen2015theory,anis2016proxies,narang2013structured}. In
order to illustrate the effects of the uncertainty in $B$, each curve
corresponds to a different value of $B$ used for \emph{estimation}
(all estimators observe the same synthetic signal of bandwidth
$B=20$). The last curve pertains to the estimator
in~\cite{narang2013structured,anis2016proxies}, which is the LS
estimator in~\eqref{eq:blestimateb} with parameter $B$ set to the
\emph{cut-off frequency} obtained from $\bm L$ and $\samplingset$ by
means of a \emph{proxy} of order 5.

Observe in Fig.~\ref{fig:bandlimited} that although the proposed MKL
estimators do not know the bandwidth, their performance is no worse
than that of the BL estimator with perfect knowledge of the signal
bandwidth. Remarkably, the MKL reconstruction schemes offer a
reasonable performance for  $\samplenum$ small, thus overcoming the need
of the LS estimator for $\samplenum\geq B$ samples.



\cmt{bandwidth estimation} The third experiment illustrates how the
bandwidth of a graph signal can be estimated using the MKL scheme from
Sec.~\ref{sec:rkhssuperposition}. To this end, a dictionary of 17
bandlimited kernels was constructed with $\beta =10^3$ and uniformly
spaced bandwidth between 10 and 90, i.e., $\fullkernelmat_\rkhsind$ is
of bandwidth $B_\rkhsind \define 5\rkhsind+5$, $\rkhsind
=1,\ldots,17$.  Fig.~\ref{fig:path} depicts the \emph{sparsity path}
for a typical realization of a bandlimited signal with bandwidth
$B=20$.  Each curve is obtained by executing
Algorithm~\ref{algo:admmc1} for different values of $\mu$ and
represents the squared modulus of the vectors
$\{\samplealphavec_\rkhsind\}_{\rkhsind=1}^\rkhsnum$
in~\eqref{eq:mklrkhsreconstruction} for a different $\rkhsind$.  As
expected, the sparsity effected in the expansion \eqref{eq:sumh}
increases with $\mu$, forcing Algorithm~\ref{algo:admmc1} to
eventually rely on a single kernel. That kernel is expected to be the
one leading to best data fit. Since the observed signal is
bandlimited, such a kernel is in turn expected to be the one in the
dictionary whose bandwidth is closest to~$B$.

Constructing a rule that determines, without human intervention, which
is the last curve $||\samplealphavec_\rkhsind||^2$ to vanish is not
straightforward since it involves comparing
$\{||\samplealphavec_\rkhsind||^2\}_{\rkhsind=1}^\rkhsnum$ for a
properly selected $\mu$.  Thus, algorithms pursuing such objective
fall out of the scope of this paper. However, one can consider the
\emph{naive} approach that focuses on a prespecified value of $\mu$ and
estimates the bandwidth as $ \hat B = B_{m^*}$, where
$m^{*}=\argmax_{\rkhsind\in\{1,\ldots,\rkhsnum\}}||\samplealphavec_\rkhsind||^2$.
Table~\ref{tab:freqestimator} reports the performance of such
estimator in terms of bias $\expectednb|B-\hat B|$ and standard
deviation $\sqrt{\expectednb|B-\expectednb\hat B|}$ for different
values of $B$ for a synthetically generated bandlimited signal.


\begin{figure}[t]
 \centering \includegraphics[width=.48\textwidth]{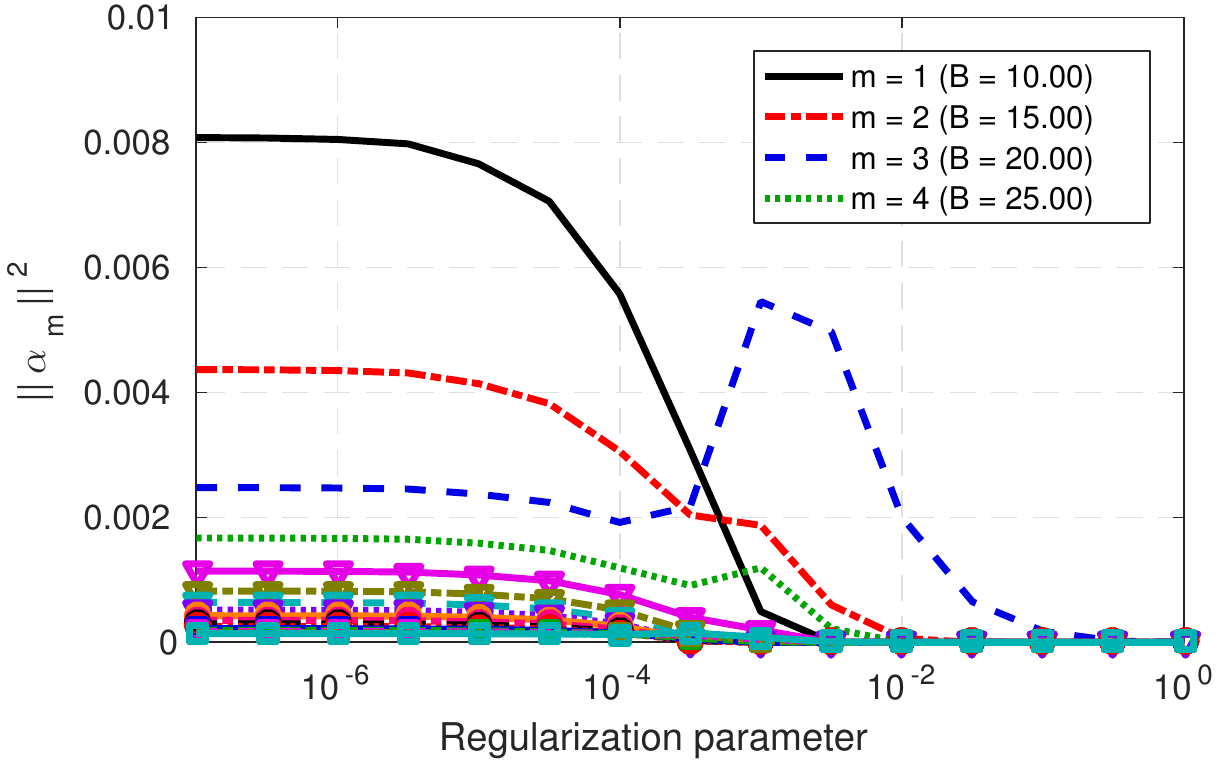}
 \caption{Sparsity path of the estimate for a typical realization.
   The legend only displays the first four curves. The last curve to
   vanish indicates the bandwidth of the observed signal
   ($\samplenum = 80$, $\vertexnum = 250$, $\text{SNR} = 20 \text{
     dB}$).}
 \label{fig:path}
\end{figure}

\begin{table}[t]
\begin{tabular}{ccccccc}
	\hline
		 & B = 10 & B = 20 & B = 30 & B = 40 & B = 50 & B = 60\\
	BIAS	 & 0.0 & 0.6 & 0.5 & 0.4 & 0.4 & 3.6\\
	STD	 & 0.0 & 1.9 & 2.9 & 1.4 & 1.4 & 10.5\\
	\hline
\end{tabular}
\caption{Bias and standard deviation for the \emph{naive} bandwidth
  estimator with $\mu=10^{-2}$ ($\samplenum = 80$, $\vertexnum = 250$,
  $\text{SNR} = 20 \text{ dB}$). 
}
\label{tab:freqestimator}
\end{table}

\begin{table*}[t!]
\center
\begin{tabular}{cccccccc}
	\hline
&	KRR with cov. kernel & Multi-kernel, RS & Multi-kernel, KS & BL for $B=2$ & BL for $B=3$ & BL, cut-off\\
	\hline
        NMSE	 & 0.34 & 0.44 & 0.43  & 1.55 & 32.64 & 3.97\\
	RMSE [min]	 & 3.95 & 4.51 & 4.45 & 8.45 & 38.72 & 13.50\\	
	\hline
\end{tabular}
\caption{Generalization NMSE and root mean square error for the
  experiment with the 
  airport data set~\cite{AirportsDataset}. }
\label{tab:airports}
\end{table*}

\subsection{Real data}

This section assesses the performance of the proposed methods with two
real-data sets. In both experiments, the data set is split into a
training set used to learn the edge weights, and a test set from which
the observations $\bm y$ are drawn for performance
evaluation. Different from the synthetic-data experiments in
Sec.~\ref{sec:syntheticdata}, where the generated noiseless function
$\signalfun$ is available and therefore the reconstruction NMSE can be
measured on observed and unobserved vertices, the experiments in this
section measure \emph{generalization} NMSE solely at unobserved
vertices.

The first data set comprises 24 signals corresponding to the average
temperature per month in the intervals 1961-1990 and 1981-2010
measured by 89 stations in
Switzerland~\cite{SwissTemperatureDataset}. The training set contains
the first 12 signals, which correspond to the interval 1961-1990,
whereas the test set contains the remaining 12.  Each station is
identified with a vertex and the graph is constructed by applying the
algorithm in~\cite{dong2015learning} with parameters $\alpha =1$ and
$\beta = 30$ to the training signals. Based on samples of a test
signal on $\samplenum$ vertices, the goal is to estimate the signal at
the remaining $\vertexnum-\samplenum$ vertices. NMSE is averaged
across the 12 test signals for a randomly chosen set
$\samplingset$. Fig.~\ref{fig:swiss} compares the performance of
the MKL schemes from Sec.~\ref{sec:mkl} along with
single-kernel ridge regression (KRR) and estimators for bandlimited
signals. The MKL algorithms employ a dictionary comprising 10
diffusion kernels with parameter $\sigma^2$ uniformly spaced between 1
and 20. Single-kernel ridge regression uses diffusion kernels for
different values of $\sigma^2$. Fig.~\ref{fig:swiss} showcases
the performance improvement arising from adopting the proposed
multi-kernel approaches.

The second data set contains departure and arrival information for
flights among U.S. airports~\cite{AirportsDataset}, from which the
$3\cdot 10^6$ flights in the months of July, August, and September of
2014 and 2015 were selected.  A graph was constructed with vertices
representing the $\vertexnum = 50$ airports with highest traffic. An
edge was placed between a pair of vertices if the number of flights
between the associated airports exceeds 100 within the observation
window. A signal was constructed per day averaging the arrival delay
of \emph{all} inbound flights per selected airport. Thus, a total of
184 signals were considered, the first 154 were used for training
(July, August, September 2014, and July, August 2015), and the
remaining 30 for testing (September 2015).

Since it is reasonable to assume that the aforementioned graph
approximately satisfies the Markovian property
(cf.~Sec.~\ref{sec:covkernels}), a Markov random field was fit to the
observations. To this end, the signals were assumed Gaussian so as to
estimate the covariance matrix of the observations via maximum
likelihood with constraints imposing  the
$(\vertexind,\vertexindp)$-th entry of the inverse covariance matrix
to be zero if $(v_\vertexind,v_\vertexindp)\notin\edgeset$. Specifically,
$\bm S\define\covmat\inv$ was found by solving the following convex
program:
\begin{align}
\begin{aligned}
\label{eq:sestimation}
  \min_{\bm S\in \rfield^{\vertexnum\times \vertexnum}}~~~&  \tr{\bm S
\covmatest\inv}-\log\det(\bm S) \\
  \text{s.to}~~~~~~& \bm S\succeq \bm 0,~ (\bm
  S)_{\vertexind,\vertexindp}=0~~~\forall(v_\vertexind,v_\vertexindp)\notin
  \edgeset
\end{aligned}
\end{align}
where $\covmatest$ is the sample covariance matrix of the training
signals after normalization to effect zero mean and unit variance per
entry of $\signalvec$. The  inverse of $\bm S$
was used as a covariance kernel (see
Sec.~\ref{sec:covkernels}). Note that such a kernel will only be
nearly optimal since the \emph{true} data covariance is unknown.

Employing Laplacian kernels or applying estimators for bandlimited
signals requires a Laplacian matrix. Although the edge set $\edgeset$
has already been constructed, it is necessary to endow those edges
with weights. Since our efforts to obtain a reasonable estimation
performance over the graphs provided by the method
in~\cite{dong2015learning} turned out unsuccessful, a novel approach
was developed. Specifically, the Laplacian matrix is sought as the
minimizer of $||\bm L - \bm S||_F^2$, where $\bm S$ is the solution to
\eqref{eq:sestimation} and $\bm L$ is a valid Laplacian with a zero at
the $(\vertexind,\vertexindp)$-th position if
$(v_\vertexind,v_\vertexindp)\notin \edgeset$.  Due to space
limitations, the rationale and details behind this approach are
skipped. 

Table~\ref{tab:airports} lists the NMSE and root mean-square error in
minutes for the task of predicting the arrival delay at 40 airports
when the delay at a randomly selected collection of 10 airports is
observed.  The second column corresponds to the ridge regression
estimator that uses the nearly-optimal \emph{estimated} covariance
kernel. The next two columns correspond to the multi-kernel approaches
in Sec.~\ref{sec:mkl} with a dictionary of 30 diffusion kernels with
values of $\sigma^2$ uniformly spaced between 0.1 and 7. The rest of
columns pertain to estimators for bandlimited signals.
Table~\ref{tab:airports} demonstrates the good performance of
covariance kernels as well as the proposed multi-kernel approaches
relative to competing alternatives.

\begin{figure}[t]
 \centering \includegraphics[width=.45\textwidth]{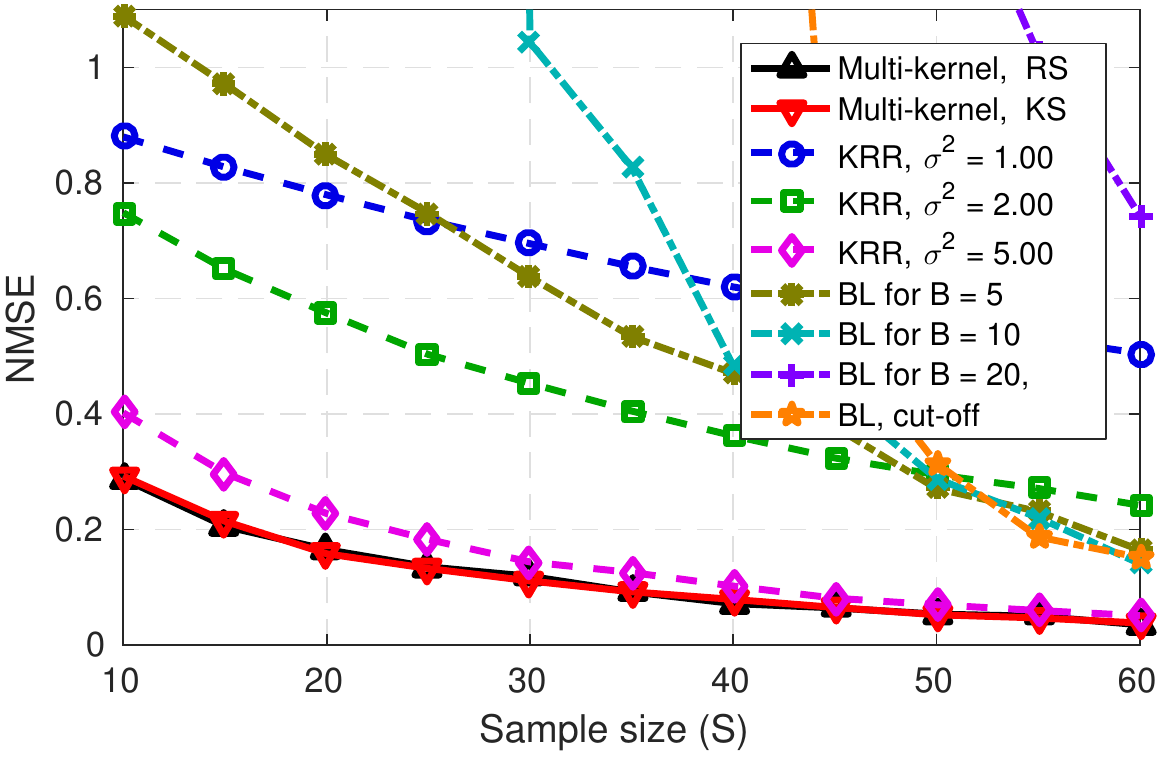}
 \caption{Generalization NMSE for the data set
   in~\cite{SwissTemperatureDataset}.}
 \label{fig:swiss}
\end{figure}


\section{Conclusions}
 \label{sec:conclusions}

 This paper introduced kernel-based learning as a unifying framework
 subsuming a number of existing signal estimators. SPoG notions such
 as bandlimitedness, graph filters, and the graph Fourier transform
 were accommodated under this perspective. The notion of bandlimited
 kernels was invoked to establish that LS estimators are limiting
 versions of the ridge regression estimator with Laplacian
 kernels. Optimality of covariance kernels was also revealed and a
 novel interpretation of kernel regression on graphs was presented in
 terms of Markov random fields. Graph filters are tantamount to
 kernel-based smoothers, which suggested applying the former to
 implement the latter in a decentralized fashion. Finally, numerical
 experiments corroborated the validity of the theoretical findings.

 Future research will pursue algorithms for learning graph Laplacian
 matrices tailored for regression, broadening regression to
 directed graphs, and numerical experiments with further data sets.




\appendices

\section{Proof of the representer theorem}
\label{appendix:representer}

\thref{prop:representer} can be proved upon decomposing $\rkhsvec$
according to the following result.
\begin{lemma}
\thlabel{lemma:rkhsdecomposition}
If $\samplingmat$ is as in Sec.~\ref{sec:kernellearning} and 
$\rkhsfun$ belongs to $\rkhs$, then $\rkhsvec \define
[\rkhsfun(v_1),\ldots,\rkhsfun(v_N)]^T$ can be expressed as
\begin{align}
\label{eq:rkhsdecomposition}
\rkhsvec = \fullkernelmat \samplingmat^T \samplealphavec + \fullkernelmat \perpvec
\end{align}
for some $\samplealphavec \in \rfield^S$ and $\perpvec\in \rfield^N$
satisfying $\samplingmat \fullkernelmat \perpvec = \bm 0$.
\end{lemma}

\begin{proof}
  Since $\rkhsfun\in \rkhs$, there exists $\fullalphavec$ such that
  $\rkhsvec = \fullkernelmat \fullalphavec$.  Thus, one needs to show
  that, for a given $\fullalphavec$, it is possible to choose
  $\samplealphavec$ and $\perpvec$ satisfying $\fullkernelmat
  \fullalphavec = \fullkernelmat \samplingmat^T \samplealphavec +
  \fullkernelmat \perpvec$ and $\samplingmat \fullkernelmat \perpvec =
  \bm 0$. This is possible, for instance, if one fixes $\perpvec =
  \fullalphavec -\samplingmat^T \samplealphavec$, and shows that there
  exists an $\samplealphavec$ such that $\samplingmat \fullkernelmat
  \perpvec = \samplingmat \fullkernelmat( \fullalphavec
  -\samplingmat^T \samplealphavec)= \bm 0$. This, in turn, follows if
  one establishes that $ \samplingmat \fullkernelmat \fullalphavec =
  \samplingmat \fullkernelmat \samplingmat^T \samplealphavec$ always
  admits a solution in $\samplealphavec$, which holds since
  $\columnspan \{\samplingmat \fullkernelmat \} = \columnspan
  \{\samplingmat \fullkernelmat \samplingmat^T \}$. To see this,
  consider the eigendecomposition $\fullkernelmat = \fullkernelevecmat
  \fullkernelevalmat \fullkernelevecmat^T$ and note that
\begin{align*}
\columnspan \{\samplingmat \fullkernelmat \samplingmat^T
\}&
=  \columnspan\{\samplingmat \fullkernelevecmat \fullkernelevalmat \fullkernelevecmat^T\samplingmat^T \}
=  \columnspan\{\samplingmat \fullkernelevecmat \fullkernelevalmat^{1/2} \}
\\&=
\columnspan\{\samplingmat \fullkernelevecmat \fullkernelevalmat
\fullkernelevecmat^T\}
=
\columnspan\{\samplingmat \fullkernelmat
\} 
\end{align*}
which concludes the proof.
\end{proof}

\thref{lemma:rkhsdecomposition} essentially states that for arbitrary
$\mathcal{S}$, any $\rkhsfun\in \rkhs$ can be decomposed into two
components as $\rkhsfun = \rkhsfun_\mathcal{S} + \rkhsfun_\perp$. The
first can be expanded in terms of the vertices indexed by
$\mathcal{S}$ as $\rkhsfun_\mathcal{S}(v) = \sum_{s=1}^S
\samplealpha_s \kernelmap(v,v_{\vertexind_s})$, whereas the second
vanishes in the sampling set, i.e.,
$\rkhsfun_\perp(v_\sampleind)=0~\forall \sampleind\in \mathcal{S}$.
Conversely, it is clear that any function that can be written as in
\eqref{eq:rkhsdecomposition} for arbitrary $\samplealphavec$ and
$\perpvec$ belonging to $\rkhs$. Hence,~\thref{lemma:rkhsdecomposition}
offers an alternative parameterization of $\rkhs$ in terms of
$\samplealphavec$ and $\perpvec$. Thus, the minimizer
$\fullalphaestvec$ of \eqref{eq:solutiongraphs} can be obtained as
$\fullalphaestvec = \samplingmat^T \samplealphaestvec + \perpestvec$,
where
\begin{align}
(\samplealphaestvec,\perpestvec) := &\argmin_{\samplealphavec,\perpvec} \frac{1}{S} 
        \mathcal{L} (\bm  y - \samplingmat \fullkernelmat (\samplingmat^T \samplealphavec +
\perpvec)) \nonumber
\\&\hspace{1.4cm}+ \regpar \Omega\left([(\samplingmat^T \samplealphavec +
\perpvec)^T \fullkernelmat (\samplingmat^T \samplealphavec +
\perpvec)]^{1/2}\right) \nonumber
\\&\text{s. to}~~\samplingmat \fullkernelmat \perpvec = \bm 0.
	\label{eq:solutiongraphsbeta}
\end{align}
Since $ \mathcal{L} (\bm y - \samplingmat \fullkernelmat
(\samplingmat^T \samplealphavec + \perpvec)) = \mathcal{L} (\bm y -
\samplingmat \fullkernelmat \samplingmat^T \samplealphavec) $, the
first term in the objective does not depend on $\perpvec$. On the
other hand, since $\Omega$ is increasing and
\begin{align*}
(\samplingmat^T \samplealphavec +
\perpvec)^T \fullkernelmat (\samplingmat^T \samplealphavec +
\perpvec) = 
\samplealphavec^T \samplingmat \fullkernelmat \samplingmat^T \samplealphavec 
+
\perpvec^T \fullkernelmat \perpvec
\end{align*}
it follows that the objective of \eqref{eq:solutiongraphsbeta} is
minimized for $\perpvec =\bm 0$, which shows that $\signalestfun$ in
\eqref{eq:single-general} can be written as $\signalestvec =
\fullkernelmat \fullalphaestvec = \fullkernelmat \samplingmat^T
\samplealphaestvec$, thus completing the proof.

\section{Big data scenarios}
\label{sec:bigdata}
Evaluating the $\vertexnum \times \vertexnum$ Laplacian kernel matrix
in~\eqref{eq:laplacian_kernel}) incurs complexity
$\mathcal{O}(\vertexnum^3)$, which does not scale well with
$\vertexnum$. This appendix explores two means of reducing this
complexity.  Both rely on solving \eqref{eq:solutiongraphsf} rather
than \eqref{eq:solutiongraphsmatrixform} since the former employs
$\fullkernelmat^\dagger = \laplacianevecmat r(\bm \Lambda)
\laplacianevecmat^T$, whereas the latter needs $\fullkernelmat$.

Recall from Sec.~\ref{sec:laplaciankernels} that Laplacian kernels
control the smoothness of an estimate by regularizing its Fourier
coefficients $|\fourierrkhsfun_\vertexind|$ via $r$. Computational
savings can be effected if one is willing to finely tune the
regularization only for large $\vertexind$, while allowing a coarse
control for small $\vertexind$. Specifically, the key idea here is to
adopt a function of the form
\begin{align}
\label{eq:rfunbd}
r(\lambda_\vertexind) = \begin{cases}
d\lambda_\vertexind & \text{if }1<\vertexind\leq B\\
d_\vertexind & \text{if }\vertexind>B \text{ or } \vertexind=1
\end{cases}
\end{align}
where $d$ and $d_\vertexind$ are constants freely selected over the
ranges $d,d_1>0$ and
$d_\vertexind>-\lambda_\vertexind$ for $\vertexind>B$. Note that
\eqref{eq:rfunbd} can be employed, in particular, to promote
bandlimited estimates of bandwidth $B$ by setting
$\{d_\vertexind\}_{\vertexind = B+1}^\vertexnum$ sufficiently large.
Defining $\oblaplacianevecmat$ as the matrix whose columns are the
$\vertexnum-B$ principal eigenvectors of $\bm L$, one obtains
\begin{align}
\fullkernelmat\inv = d\bm L + \oblaplacianevecmat
(
\bm \Delta - d\oblaplacianevalmat
)
\oblaplacianevecmat^T + d_1\bm 1 \bm 1^T + \epsilon \bm I_N
\end{align}
where $\bm \Delta \define \diag{d_{B+1},\ldots,d_\vertexnum}$ and $
\epsilon \bm I_N$ with $\epsilon>0$ is added to ensure that
$\fullkernelmat$ is invertible in case that the multiplicity of the
zero eigenvalue of $\bm L$ is greater than one, which occurs when the
graph has multiple connected components.

 Alternative functions that do not require eigenvector
computation are low-order polynomials of the form
\begin{align}
r(\lambda) = \sum_{p=0}^P a_p \lambda^p.
\end{align}
In this case, the resulting $\fullkernelmat\inv$  reads as
\begin{align*}
\fullkernelmat\inv = a_0 \bm I_N + \sum_{p=1}^P a_p \bm L^p.
\end{align*}
The cost of obtaining this matrix is reduced since  powers of $\bm
L$ can be efficiently computed when $\bm L$ is sparse, as is typically
the case. In the extreme case where $P=1$, $a_1>0$, and
$a_0\rightarrow 0$, the regularizer becomes $\rkhsvec^T \bm L
\rkhsvec$, which corresponds to the Laplacian regularization
(cf. Sec.~\ref{sec:laplaciankernels}).

\section{Proof of \thref{th:blandkernel}}
\label{sec:proof:th:blandkernel}

 Without loss of generality, let $\blset = \{1,\ldots,B\}$;
otherwise, simply permute the order of the eigenvalues. Define also
the $\vertexnum \times B$ matrix $\blselectionmat = [\bm I_B, \bm
  0]^T$ and the $\vertexnum \times (\vertexnum-B)$ matrix
$\compblselectionmat = [ \bm 0,\bm I_{\vertexnum-B}]^T$, whose
concatenation clearly gives $[\blselectionmat, \compblselectionmat] =
\bm I_{N}$. 
 Since in this case $\bllaplacianevecmat =
\laplacianevecmat \blselectionmat$, \eqref{eq:blestimateb} becomes
\begin{align}
\label{eq:blandkernelbl}
  \signallsestvec
&=
\laplacianevecmat \blselectionmat
[ \blselectionmat^T\laplacianevecmat^T\samplingmat^T\samplingmat\laplacianevecmat \blselectionmat]\inv
 \blselectionmat^T\laplacianevecmat^T\samplingmat^T\observationvec .
\end{align}

 On the other hand, the ridge regression version of
\eqref{eq:solutiongraphsf} is 
\begin{align}
\label{eq:alternativeridge}
\signalestvec &:= \argmin_{\rkhsvec} \frac{1}{S} 
        ||\observationvec - \samplingmat\rkhsvec||^2
+ \regpar \rkhsvec^T \fullkernelmat\inv \rkhsvec
\end{align}
where the constraint has been omitted since
$r_\beta(\lambda_\vertexind)>0~\forall \vertexind$.  The minimizer of
\eqref{eq:alternativeridge} is 
\begin{subequations}
\begin{align}
\signalestvec &=
 (\samplingmat^T\samplingmat  +
\mu S \fullkernelmat\inv)\inv \samplingmat^T\observationvec 
\\
&= \laplacianevecmat (\laplacianevecmat^T\samplingmat^T\samplingmat \laplacianevecmat +
\mu S r_\beta(\laplacianevalmat))\inv
\laplacianevecmat^T\samplingmat^T\observationvec.
\label{eq:blandkernelkernel}
\end{align}
\end{subequations}

Establishing that $\signalestvec \rightarrow \signallsestvec $ therefore
amounts to showing that the right-hand side of
\eqref{eq:blandkernelbl} converges to that  of
\eqref{eq:blandkernelkernel}.  For this, it  suffices to prove that
\begin{align}
\label{eq:blandkernelsufficient}
&(\auxmat +
\mu S r_\beta(\laplacianevalmat))\inv
\rightarrow
 \blselectionmat
[ \blselectionmat^T\auxmat \blselectionmat]\inv
 \blselectionmat^T
\end{align}
where $\auxmat \define \laplacianevecmat^T\samplingmat^T\samplingmat
\laplacianevecmat$.  Note that $\blselectionmat^T\auxmat
\blselectionmat =
\bllaplacianevecmat^T\samplingmat^T\samplingmat\bllaplacianevecmat $
is invertible by hypothesis.  With $\iblaplacianevalmat\define
(1/\beta)\bm  I_B$ and $\oblaplacianevalmat\define \beta\bm I_{\vertexnum-B}$
representing the \emph{in-band} and \emph{out-of-band} parts of
$r_\beta(\laplacianevalmat)$, the latter can be written as
$r_\beta(\laplacianevalmat) =
\diag{\iblaplacianevalmat,\oblaplacianevalmat}$. With this notation,
 \eqref{eq:blandkernelsufficient} becomes
\begin{align}
\label{eq:blandkernelsufficientmat}
&\nonumber
\left(
\left[
\begin{array}{c c}
\blselectionmat^T \auxmat \blselectionmat & \blselectionmat^T \auxmat
\compblselectionmat\\
\compblselectionmat^T \auxmat \blselectionmat & \compblselectionmat^T \auxmat
\compblselectionmat
\end{array}
\right]
+
\mu S\left[
\begin{array}{c c}
\iblaplacianevalmat &\bm 0\\
\bm 0 &\oblaplacianevalmat \\
\end{array}
\right]
\right)\inv\\
&\rightarrow
\left[
\begin{array}{c c}
( \blselectionmat^T\auxmat \blselectionmat)\inv &\bm 0\\
\bm 0 & \bm 0
\end{array}
\right].
\end{align}
 Using block matrix inversion formulae,  it readily follows that the
left-hand side equals the following matrix product
\begin{align}
\nonumber
&\left[
\begin{array}{c c}
\bm I_B &
-(\blselectionmat^T \auxmat \blselectionmat + \mu S \iblaplacianevalmat)\inv \blselectionmat^T \auxmat
\compblselectionmat
\\
\bm 0 & \bm I_{\vertexnum-B}
\end{array}
\right]
\\&
\left[
\begin{array}{c c}
(\blselectionmat^T \auxmat \blselectionmat + \mu S \iblaplacianevalmat)\inv
 &\bm 0\\
\bm 0 & \auxmatt\inv
\end{array}
\right]
\label{eq:blandkernelleftterm}
\\&
\left[
\begin{array}{c c}
\bm I_B &
\bm 0\\
-
\compblselectionmat^T \auxmat\blselectionmat
(\blselectionmat^T \auxmat \blselectionmat + \mu S \iblaplacianevalmat)\inv 
 & \bm I_{\vertexnum-B}
\end{array}
\right]
\nonumber
\end{align}
where
\begin{align*}
\auxmatt&\define
\compblselectionmat^T
 \auxmat
\compblselectionmat + \mu S \oblaplacianevalmat
\\&-
\compblselectionmat^T \auxmat \blselectionmat
(\blselectionmat^T \auxmat \blselectionmat + \mu S \iblaplacianevalmat)\inv \blselectionmat^T \auxmat
\compblselectionmat.
\end{align*}
 Recalling that $\blselectionmat^T\auxmat \blselectionmat$ is
invertible and letting $\beta \rightarrow \infty$, it follows that
$\auxmatt\inv \rightarrow \bm 0$ and $\iblaplacianevalmat\rightarrow
\bm 0$ as $\beta \rightarrow \infty$, which implies that
\eqref{eq:blandkernelleftterm} converges to the right-hand side of
\eqref{eq:blandkernelsufficientmat} and concludes the proof.

\section{Proof of \thref{prop:markov}}
\label{app:proof:prop:markov}

 The first-order optimality condition for
\eqref{eq:single-lq-m-a} is given by
\begin{equation}
\label{eq:covariancekernelsoptimality}
 \noisevar \covmat\inv\rkhsvec =
\samplingmat^T( \observationvec-\samplingmat \rkhsvec).
\end{equation}
Without loss of generality, one can focus on the relation implied by
the first row of \eqref{eq:covariancekernelsoptimality}.  To this end,
partition $\covmat$ as
\begin{align*}
\covmat =
\left[
\begin{array}{c c}
\cov_{1,1}&\covvec_{\rest,1}^T\\
\covvec_{\rest,1}&\covmat_{\rest,\rest}
\end{array}
\right]
\end{align*}
and apply block matrix inversion formulae to obtain
\begin{align*}
\covmat\inv =
\frac{1}{\covcond}
\left[
\begin{array}{c c}
1&-\covvec_{\rest,1}^T\covmat_{\rest,\rest}\inv\\
-\covmat_{\rest,\rest}\inv\covvec_{\rest,1}&\tilde\covmat\inv
\end{array}
\right]
\end{align*}
where  $\covcond\define
\cov_{1,1}-\covvec_{\rest,1}^T\covmat_{\rest,\rest}\inv\covvec_{\rest,1}$
and
 \begin{align*}
\tilde \covmat\inv\define\covmat_{\rest,\rest}\inv
\covvec_{\rest,1}
\covvec_{\rest,1}^T\covmat_{\rest,\rest}\inv
+\covcond
\covmat_{\rest,\rest}\inv.
\end{align*}
Note that $\sigma^2_{\vertexind|\mathcal{N}_\vertexind} \define
\covcond $ is in fact the variance of the LMMSE predictor for
$\signalfun(v_1)$ given
$\signalfun(v_2),\ldots,\signalfun(v_\vertexnum)$.  

Two cases can be considered for the first row
of~\eqref{eq:covariancekernelsoptimality}. First, if $1\notin
\samplingset$, then the first row of $\samplingmat$ is zero, and the
first row of \eqref{eq:covariancekernelsoptimality} becomes
\begin{subequations}
\begin{align}
\label{eq:rowonezero}
\rkhsfun(v_1)&=\covvec_{\rest,1}^T\covmat_{\rest,\rest}\inv\rkhsvec_{\rest} \\&= 
\sum_{n:v_n\in \mathcal{N}_1} (-\covcond
\invcovmatent_{1,n}) \rkhsfun(v_n) \label{eq:rowonezero2}
\end{align}
\end{subequations}
where $\rkhsvec_{\rest}\define[f(v_2),\ldots,f(v_\vertexnum)]^T$ and
$\invcovmatent_{\vertexind,\vertexindp} =
(\covmat\inv)_{\vertexind,\vertexindp}$. The sum in
\eqref{eq:rowonezero2} involves only the neighbors of $v_1$ since the
graph is a Markov random field, for which if there is no edge between
$v_\vertexind$ and $v_\vertexindp$, then $\signalfun(v_\vertexind)$
and $\signalfun(v_\vertexind)$ are conditionally independent given the
rest of vertices, which in turn implies that
$\invcovmatent_{\vertexind,\vertexindp}=0$.
Note that the right-hand side of \eqref{eq:rowonezero} is the LMMSE
predictor of $\rkhsfun(v_1)$ given the estimated function value at its
neighbors. Since this argument applies to all vertices
$v_\vertexind,~\vertexind\notin \samplingset$, it follows that the
optimality condition \eqref{eq:covariancekernelsoptimality} seeks
values of $\rkhsfun(v_\vertexind)$ so that the function value at unobserved
vertices agrees with its LMMSE estimate given the estimated value at
its neighbors.

 On the other hand, if $1\in \samplingset$, the first row of
$\samplingmat$ has a 1 at the $(1,1)$ position, which implies that the
first row of \eqref{eq:covariancekernelsoptimality} is
\begin{align*}
 \observation_1 &=
\rkhsfun(v_1)+\frac{\noisevar}{\covcond}(\rkhsfun(v_1)-\covvec_{\rest,1}^T\covmat_{\rest,\rest}\inv\rkhsvec_{\rest}
).
\end{align*}
The second term on the right can be thought of as an estimate of the
noise $\noisesamp_1$ present in $\observation_1$.  Therefore, the
optimality condition imposes that each observation
$\observation_{\sampleind(\vertexind)}$ agrees with the estimated
noisy version of the function given the neighbors of~$v_\vertexind$.

\bibliographystyle{IEEEbib}
\bibliography{my_bibliography}

\end{document}